\documentclass[10pt]{article}

\usepackage{hyperref,psfrag}
\usepackage[margin=1in]{geometry}

\usepackage{amsmath,amsbsy,amsfonts,amssymb,amsthm,dsfont,units,bm}
\usepackage{graphicx,wrapfig}
\usepackage{authblk}
\usepackage{color,cases}
\usepackage{tikz}
\usetikzlibrary{calc,shapes}

\newcommand{\abs}[1]{\left\lvert #1 \right\rvert}

\usepackage{url,cite}
\usepackage[square,sort,comma,numbers]{natbib}
\usepackage[outdir=./]{epstopdf}
\usepackage{graphicx,float,pgfplots,wrapfig,sidecap,lipsum}
\usepackage{tabularx}
\usepackage{booktabs}
\usepackage{algorithm}
\usepackage[noend]{algorithmic}
\usepackage{mathtools}

\def\rinfnorm{\rVert_{\infty}}

\DeclarePairedDelimiter\infnorm{\lVert}{\rinfnorm}
\newtheorem{prop}{Proposition}[section]
\newtheorem{lem}{Lemma}[section]

\DeclareMathOperator{\mat}{matrix}
\DeclareMathOperator{\diag}{diag}

\newcommand\eps{\epsilon}

\newcommand\teps{\tl{\eps}}

\newcommand\tlambda{\ensuremath{\tilde{\lambda}}}
\newcommand\tlambdamax{\ensuremath{\tilde{\lambda}_{\max}}}
\newcommand\tlambdamin{\ensuremath{\tilde{\lambda}_{\min}}}

\renewcommand\th[1]{\ensuremath{\theta_{#1}}}
\newcommand\ut[1]{\ensuremath{\check{\theta}_{#1}}}

 \newcommand{\IGNORE}[1]{}

\newcommand{\ignore}[1]{}
\renewcommand{\cal}[1]{\mathcal #1}

\newcommand{\R}{\mathbb R}

\newcommand{\twonorm}[1]{\left\lVert #1 \right\rVert_{2}}
\newcommand{\frob}[1]{\left\lVert #1 \right\rVert_{F}}

\newcommand{\Sec}[1]{\hyperref[sec:#1]{\S\ref*{sec:#1}}} 
\newcommand{\Eqn}[1]{\hyperref[eq:#1]{(\ref*{eq:#1})}} 
\newcommand{\Fig}[1]{\hyperref[fig:#1]{Fig.\,\ref*{fig:#1}}} 
\newcommand{\Tab}[1]{\hyperref[tab:#1]{Tab.\,\ref*{tab:#1}}} 
\newcommand{\Thm}[1]{\hyperref[thm:#1]{Theorem\,\ref*{thm:#1}}} 
\newcommand{\Lem}[1]{\hyperref[lem:#1]{Lemma\,\ref*{lem:#1}}} 
\newcommand{\Prop}[1]{\hyperref[prop:#1]{Prop.~\ref*{prop:#1}}} 
\newcommand{\Cor}[1]{\hyperref[cor:#1]{Corollary~\ref*{cor:#1}}} 
\newcommand{\Def}[1]{\hyperref[def:#1]{Definition~\ref*{def:#1}}} 
\newcommand{\Alg}[1]{\hyperref[alg:#1]{Alg.~\ref*{alg:#1}}} 
\newcommand{\Ex}[1]{\hyperref[ex:#1]{Ex.~\ref*{ex:#1}}} 
\newcommand{\Clm}[1]{\hyperref[clm:#1]{Claim~\ref*{clm:#1}}} 
\colorlet{shadecolor}{blue!20}

\newcommand{\defas}{:=}

\newcommand{\Etn}{\E^{(t+1)}}

\renewcommand\t{{\scriptscriptstyle\top}}

\DeclareMathOperator{\ncralgo}{RTD}

\newcommand\tl{\tilde}

 \renewcommand{\L}{L}
\renewcommand{\S}{S}

\newcommand{\Lo}{\L^*}
\newcommand{\So}{\S^*}

\newcommand{\Lt}[1][t]{\L^{(#1)}}
\newcommand{\St}[1][t]{\S^{(#1)}}

\newcommand{\Ltn}{\L^{(t+1)}}
\newcommand{\Stn}{\S^{(t+1)}}

\newcommand{\M}{T}

\newcommand{\E}{{E}}

\newcommand{\Et}[1][t]{\E^{(#1)}}

\def\tl{\tilde}\renewcommand\t{{\scriptscriptstyle\top}}

\newcommand\inner[1]{\ensuremath{\langle #1 \rangle}}

\DeclareMathOperator*{\HT}{\mathcal{H}}

\DeclareMathOperator{\supp}{supp}

 \DeclareMathOperator*{\argmax}{arg\,max}

\DeclareMathOperator{\dist}{dist}

\def\tha{{\mbox{\tiny th}}}

 \def\0{{\bf 0}}

\def\viz{{viz.,\ \/}}

\def\nn{\nonumber}

\def\qed{\hfill\hbox{${\vcenter{\vbox{
    \hrule height 0.4pt\hbox{\vrule width 0.4pt height 6pt
    \kern5pt\vrule width 0.4pt}\hrule height 0.4pt}}}$}}

\def\Nc{{\cal N}}

\def\Rbb{{\mathbb R}}

\newcommand{\bprf}{\begin{myproof}}
\newcommand{\eprf}{\end{myproof}}
\newcommand{\bp}{\begin{psfrags}}
\newcommand{\ep}{\end{psfrags}}
\newcommand{\bl}{\begin{lemma}}
\newcommand{\el}{\end{lemma}}
\newcommand{\bt}{\begin{theorem}}
\newcommand{\et}{\end{theorem}}
\newcommand{\bc}{\begin{center}}
\newcommand{\ec}{\end{center}}
\newcommand{\bi}{\begin{itemize}}
\newcommand{\ei}{\end{itemize}}
\newcommand{\ben}{\begin{enumerate}}
\newcommand{\een}{\end{enumerate}}
\newcommand{\bd}{\begin{definition}}
\newcommand{\ed}{\end{definition}}
\def\beq{\begin{equation}}
\def\eeq{\end{equation}\noindent}
\def\beqn{\begin{eqnarray}}
\def\eeqn{\end{eqnarray} \noindent}
\def\beqnn{  \begin{eqnarray*}}
\def\eeqnn{\end{eqnarray*}  \noindent}
\def\bcase{  \begin{numcases}}
\def\ecase{\end{numcases}   \noindent}
\def\bsbcase{  \begin{subnumcases}}
\def\esbcase{\end{subnumcases}   \noindent}

\newtheorem{theorem}{Theorem}
\newtheorem{corollary}{Corollary}
\newtheorem{lemma}[theorem]{Lemma}

\newtheorem{definition}{Definition}

\newenvironment{myproof}{\noindent{\bf Proof:} \hspace*{1em}}{
    \hspace*{\fill} $\Box$ }
\newenvironment{proof_of}[1]{\noindent {\bf Proof of #1: }}{\hspace*{\fill} $\Box$ }

\newcommand{\matplottc}[1]{               
        \unitlength .45truein
        \begin{center}
        \includegraphics{#1.ps}
        \end{picture}
        \end{center}
}

\def\psfancypar#1#2{\begingroup\def\par{\endgraf\endgroup\lineskiplimit=0pt}
               \setbox2=\hbox{\large\sc #2}
               \newdimen\tmpht \tmpht \ht2 \advance\tmpht by \baselineskip
               \font\hhuge=Times-Bold at \tmpht
               \setbox1=\hbox{{\hhuge #1}}
               \count7=\tmpht \count8=\ht1
               \divide\count8 by 1000 \divide\count7 by \count8
               \tmpht=.001\tmpht\multiply\tmpht by \count7
               \font\hhuge=Times-Bold at \tmpht
               \setbox1=\hbox{{\hhuge #1}}
               \noindent
                \hangindent1.05\wd1
               \hangafter=-2 {\hskip-\hangindent
               \lower1\ht1\hbox{\raise1.0\ht2\copy1}%
                \kern-0\wd1}\copy2\lineskiplimit=-1000pt}

\def\Kout{\setbox1=\hbox{\Huge\bf K}\hbox to
1.05\wd1{\hspace{.05\wd1}
}}
\def\Sout{\setbox1=\hbox{\Huge\bf S}\hbox to 1.05\wd1{\hspace{.05\wd1}
}}

\newcommand{\torestate}[3]{%
\expandafter \def \csname BBRESTATE #2 \endcsname{#3}
\theoremstyle{plain}
\newtheorem{BBRESTATETHMNUM#2}[theorem]{#1}
\begin{BBRESTATETHMNUM#2}\label{#2}\csname BBRESTATE #2 \endcsname   \end{BBRESTATETHMNUM#2}
\newtheorem*{BBRESTATETHMNONNUM#2}{{#1}~\ref{#2}}
}

\newcommand{\restate}[1]{\begin{BBRESTATETHMNONNUM#1}[Restated] \csname BBRESTATE #1 \endcsname
\end{BBRESTATETHMNONNUM#1}}

\title{Tensor vs Matrix Methods: \\ Robust Tensor Decomposition under Block Sparse Perturbations}

\date{\today}

\author{Animashree Anandkumar\thanks{University of California, Irvine. Email: a.anandkumar@uci.edu}, Prateek Jain\thanks{Microsoft Research, India. Email: prajain@microsoft.com}, Yang Shi\thanks{University of California, Irvine. Email: shiy4@uci.edu}, U. N. Niranjan\thanks{University of California, Irvine. Email: un.niranjan@uci.edu}}

\begin{document}

\maketitle

\begin{abstract}
Robust tensor CP decomposition involves decomposing  a tensor into low rank and sparse components. We propose a novel non-convex iterative algorithm with guaranteed recovery.  It alternates  between  low-rank CP decomposition through gradient ascent (a variant of the tensor power method), and  hard thresholding of the residual. We prove convergence to the globally optimal solution   under natural incoherence conditions on the low rank component, and bounded level of sparse perturbations. We compare our method with natural baselines, \viz which apply   robust  matrix   PCA either to the {\em flattened} tensor, or to the matrix slices of the tensor. Our method can provably handle a far greater level of perturbation  when the sparse  tensor is  block-structured. 
Thus, we establish that tensor methods can tolerate  a higher level of gross corruptions compared to matrix methods.
\end{abstract}

\section{Introduction}
In this paper, we develop a robust tensor decomposition  method, which  recovers a low rank  tensor   subject to gross corruptions.  Given an input tensor $T=L^*+S^*$, we aim to recover both $L^*$ and $S^*$, where $\Lo$ is a low rank tensor and $\So$ is a sparse tensor
\begin{equation}\label{eqn:robust-def} \M = \Lo+\So,\quad \Lo=\sum_{i=1}^r \sigma_i^* u_i\otimes u_i \otimes u_i\end{equation} and $\M, \Lo, \So\in \Rbb^{n \times n \times n}$. The above form of $\Lo$  is known as the Candecomp/Parafac or the CP-form. We assume that $\Lo$ is a rank-$r$ orthogonal tensor, i.e., $\langle u_i, u_j\rangle=1$ if $i=j$ and $0$ otherwise. The above problem arises in numerous applications such as image and video denoising~\cite{huang2008robust}, multi-task learning, and robust learning of latent variable models (LVMs) with grossly-corrupted moments, for  details see Section~\ref{sec:app}.

The matrix version of \eqref{eqn:robust-def}, \viz decomposing a matrix into sparse and low rank matrices,   is known as    robust principal component analysis (PCA).  It  has been studied extensively~\cite{chandrasekaran2011rank,candes2011robust,hsu2011robust,netrapalli2014non}. Both convex ~\cite{chandrasekaran2011rank,candes2011robust} as well as non-convex~\cite{netrapalli2014non} methods have been proposed with provable recovery. 

One can attempt to solve  the robust tensor problem  in \eqref{eqn:robust-def}  using matrix methods. In other words, robust matrix PCA can be applied either to  each matrix slice of the tensor, or to the matrix obtained by flattening the tensor. However, such matrix methods ignore the  tensor algebraic constraints or the CP  rank constraints, which differ from the matrix rank constraints. There are however a number of challenges to incorporating the tensor CP rank   constraints. Enforcing a given  tensor rank   is NP-hard~\cite{hillar2013most}, unlike the matrix case, where low rank projections can be computed efficiently. Moreover, finding the best convex relaxation of the tensor CP rank is also NP-hard~\cite{hillar2013most}, unlike   the matrix case, where the convex relaxation of the rank, \viz the nuclear norm,  can be computed efficiently.

\subsection{Summary of Results}
\vspace*{-5pt}
\paragraph{Proposed method: }We   propose a  non-convex iterative method, termed $\ncralgo$, that maintains low rank  and sparse estimates $\hat{L}$, $\hat{S}$, which are alternately updated. The low rank estimate $\hat{L}$ is updated through the eigenvector computation of $T-\hat{S}$, and the sparse estimate is updated through (hard) thresholding of the residual $T-\hat{L}$.
\vspace*{-5pt}
\paragraph{Tensor Eigenvector Computation: } Computing eigenvectors of $T-\hat{S}$ is challenging as the tensor can have arbitrary ``noise'' added to an orthogonal tensor and hence the techniques of~\cite{AnandkumarEtal:tensor12} do not apply as they only guarantee an approximation to the eigenvectors up to the ``noise'' level. Results similar to the shifted power method of~\cite{DBLP:journals/siammax/KoldaM11} should apply for our problem, but their results hold in an arbitrarily small centered at the true eigenvectors, and the size of the ball is typically not well-defined. In this work, we provide a simple variant of the tensor power method based on gradient ascent of a regularized variational form of the eigenvalue problem of a tensor. We show that our method converges to the true eigenvectors at a linear rate when initialized within a reasonably small ball around eigenvectors. See Theorem~\ref{thm:robustpower} for details. 
\vspace*{-5pt}
\paragraph{Guaranteed recovery: }As a main result, we prove convergence to the global optimum $\{\Lo, \So\}$ for $\ncralgo$  under an  incoherence assumption on $\Lo$, and a bounded sparsity level for   $\So$.  These conditions are similar to the conditions required for the success of matrix robust PCA. We also prove fast linear convergence rate for $\ncralgo$, i.e. we obtain  an additive $\epsilon$-approximation in $O(\log(1/\epsilon))$ iterations.
\vspace*{-5pt}
\paragraph{Superiority over matrix robust PCA: }We  compare our $\ncralgo$ method with matrix robust PCA, carried out either on matrix slices of the tensor, or on the {\em flattened} tensor. We prove our $\ncralgo$ method is superior and can handle higher sparsity levels in the noise tensor $\So$, when it is block structured. Intuitively, each block of noise represents correlated noise which persists for a subset  of slices in the tensor. For example, in a video if there is an occlusion then the occlusion remains fixed in a small number of frames. In the scenario of  moment-based estimation,  $\So$ represents gross corruptions of the moments of some multivariate distribution, and we can assume that it occurs over a small subset of variables.

We prove that our tensor methods can handle a much higher level of block sparse perturbations, when the overlap between the blocks is controlled (e.g.  random block sparsity). For example, for a rank-$1$ tensor, our method can handle $O(n^{17/12})$ corrupted entries per fiber of the tensor (i.e. row/column of a slice of the tensor). In contrast,   matrix robust PCA   methods  only allows for $O(n)$ corrupted entries, and this bound is tight \cite{netrapalli2014non}. We prove that even better  gains are obtained for $\ncralgo$ when the rank $r$ of $\Lo$ increases, and we provide precise results in this paper. Thus, our $\ncralgo$ achieves best of both the worlds: better accuracy and faster running times.

We conduct extensive simulations to empirically validate the performance of our method and compare it to various matrix robust PCA methods.    Our  synthetic experiments  show that our tensor method is   $2$-$3$ times more accurate, and about $8$-$14$ times faster, compared to  matrix decomposition methods. On the real-world  \textit{Curtain} dataset, for the activity detection, our tensor method obtains better recovery with a $10$\% speedup.
\vspace*{-5pt}
\paragraph{Overview of techniques: }At a high level, the proposed method $\ncralgo$ is a  tensor analogue of the non-convex  matrix robust PCA method in~\cite{netrapalli2014non}. However, both the algorithm ($\ncralgo$) and the analysis of  $\ncralgo$ is significantly challenging due to two key reasons: a) there can be significantly more structure in the tensor problem that needs to be exploited carefully using structure in the noise,
b) unlike matrices, tensors can have an exponential number of eigenvectors~\cite{cartwright2013number}.

We would like to stress that we need to establish convergence to the globally optimal solution $\{\Lo, \So\}$,  and not just to a local optimum, despite the non-convexity of the decomposition problem. Intuitively, if we are in the basin of attraction of the global optimum, it is natural to expect that  the estimates $\{\hat{L}, \hat{S}\}$  under $\ncralgo$ are progressively   refined, and get closer to the true solution $\{\Lo, \So\}$. However, characterizing this basin, and the  conditions needed to ensure we ``land'' in this basin   is non-trivial and novel.

As mentioned above, our method alternates between finding low rank estimate $\hat{L}$ on the residual $T-\hat{S}$ and viceversa. The main steps in our proof are as follows: {\em (i)} For updating the low rank estimate, we propose a modified tensor power method, and prove that it converges to one of the eigenvectors of $T-\hat{S}$. In addition,  the recovered eigenvectors are ``close'' to the components of $\Lo$. {\em (ii)} When the sparse estimate  $\hat{S}$ is updated through hard thresholding, we prove that the support of $\hat{S}$ is contained within that of $\So$. {\em (iii)} We make strict progress in each epoch, where $\hat{L}$ and $\hat{S}$ are alternately updated.

In order to prove the first part, we establish that the proposed method performs gradient ascent on a regularized variational form of the eigenvector problem. We then establish that the regularized objective satisfies local strong convexity and smoothness. We also establish that by having a polynomial number of initializations, we can recover vectors that are ``reasonably'' close to eigenvectors of $T-\hat{S}$. Using the above two facts, we establish a linear convergence to the true eigenvectors of $T-\hat{S}$, which are close to the components of $\Lo$.

For step {\em (ii)} and {\em (iii)}, we  show that using an intuitive block structure in the noise, we can bound the affect of noise on the eigenvectors of the true low-rank tensor $\Lo$ and show that the proposed iterative scheme refines the estimates of $\Lo$ and converge to $\Lo$ at a linear rate.

\subsection{Applications}\label{sec:app}
In addition to the standard applications of robust tensor decomposition to video denoising~\cite{huang2008robust}, we propose two applications in probabilistic learning.
\vspace*{-5pt}
\paragraph{Learning latent variable models using  grossly corrupted moments: }Using tensor decomposition for learning LVMs has been intensely studied in the last few years, e.g.~\cite{AnandkumarEtal:tensor12,AnandkumarEtal:community12COLT,jain2013learning}. The idea is to learn the models through CP decomposition of higher order moment tensors.
While the above works assume access to empirical moments, we can extend the framework to that of {\em robust} estimation, where the moments are subject to gross corruptions. 
In this case, gross corruptions on the moments can occur either due to adversarial manipulations or systematic bias in estimating moments of some subset of variables.
\vspace*{-5pt}
\paragraph{Multi-task learning of linear Bayesian networks: }
Let $z_{(i)} \sim \Nc (0, \Sigma_{(i)})$. The samples for the $i^{\tha}$ Bayesian network  are generated as $
x_{(i)} = U h_{(i)} + z_{(i)},$ and  $
\mathbb{E}[x_{(i)} x_{(i)}^\top] = U \diag (w_{(i)}) U^\top + \Sigma_{(i)},
$ where $h_{(i)}$ is the hidden variable for the $i^{\tha}$ network. If the Bayesian networks are related, we can share parameters among them. In the above framework,  we share   parameters $U$, which   map the hidden variables  to observed ones.
Assuming that all the  covariances
$\Sigma_{(i)}$ are sparse, when they are stacked together, they   form a sparse tensor. Similarly $U\diag (w_{(i)}) U^\top$ stacked  constitutes a low rank tensor. Thus, we can consider the samples jointly, and learn the parameters by performing robust tensor   decomposition.

\subsection{Related Work}\label{sec:rel}
\vspace*{-5pt}
\paragraph{Robust matrix decomposition: }
In the matrix setting, the above problem of decomposition into sparse and low rank parts is popularly known as {\em robust PCA}, and has been studied in a number of works (\cite{chandrasekaran2011rank},\cite{candes2011robust}). The popular method is based on convex relaxation, where the low rank penalty is replaced by nuclear norm and the sparsity is replaced by the $\ell_1$ norm. However, this technique is not applicable in the tensor setting, when we consider the {\em CP rank}. There is no convex surrogate available for the CP rank.

Recently a non convex method for robust PCA is proposed in~\cite{netrapalli2014non}. It involves alternating steps of PCA and   thresholding of the residual. Our proposed tensor method can be seen as a tensor analogue of the method in~\cite{netrapalli2014non}. However, the analysis is very different, since the optimization landscape for tensor decomposition differs significantly from that of the matrix.
\vspace*{-5pt}
\paragraph{Convex Robust Tucker decomposition: }
Previous works which employ convex surrogates for tensor problems employ a different notion of rank, known as the {\em Tucker} rank or the {\em multi-}rank, e.g.~\cite{tomioka2010estimation,gandy2011tensor,huang2014provable,kreimer2013tensor,NIPS2014_5409}. However, the notion of a multi-rank is based on ranks of the matricization or flattening of the tensor, and thus, this method does not exploit the tensor algebraic constraints.  The problem of robust tensor PCA is specifically tackled in~\cite{gu2014robust,goldfarb2014robust}. In~\cite{goldfarb2014robust}, convex and non-convex methods are proposed based on {\em Tucker} rank, but there are no guarantees on if it yields the original tensors $\Lo$ and $\So$ in \eqref{eqn:robust-def}. In ~\cite{gu2014robust}, they prove success under restricted eigenvalue conditions. However, these conditions are opaque and it is not clear regarding the level of sparsity that can be handled. 
\vspace*{-5pt}
\paragraph{Sum of squares: }Barak et al~\cite{Boaz:2015} recently consider  CP-tensor completion using algorithms based on the sum of squares hierarchy. However, these algorithms are expensive. 
In contrast, in this paper, we consider simple iterative methods based on the power method that are efficient and scalable for large datasets. It is however unclear if the sum of squares algorithm improves the result for the block sparse model considered here.
\vspace*{-5pt}
\paragraph{Robust tensor decomposition: }Shah et al~\cite{shah2015sparse} consider robust tensor decomposition method using a randomized convex relaxation formulation. Under their random sparsity model, their algorithm provides guaranteed recovery as long as the number of non-zeros per fibre is $O(\sqrt{n})$. This is in contrast to our method which potentially tolerates upto $O(n^{17/12})$ non-zero sparse corruptions per fibre.

\section{Proposed Algorithm}\label{sec:algo}
\vspace*{-5pt}
\paragraph{Notations: } Let $[n] := \{1,2,\dotsc,n\}$, and $\|v\|$ denote the $\ell_2$ norm of vector $v$. For a matrix or a tensor $M$,  $\|M\|$ refers to spectral norm   and $\|M\|_\infty $ refers to maximum absolute entry.
\vspace*{-5pt}
\paragraph{Tensor preliminaries: }
A real third order tensor $T \in \R^{n \times n \times n}$ is a three-way array.
The different dimensions of the tensor are referred to as {\em modes}. 
In addition, {\em fibers} are higher order analogues of matrix rows and columns. A fiber is obtained by fixing all but one of the indices of the tensor. For example, for a third order tensor $T\in \R^{n \times n \times n}$, the mode-$1$ fiber is given by $T(:, j, l)$. 
Similarly, {\em slices} are obtained by fixing all but two of the indices of the tensor. For example, for the third order tensor $T$, the slices along third mode are given by $T(:, :, l)$. A {\em flattening} of tensor $T$ along mode $k$ is a matrix $M$ whose columns correspond to mode-$k$ fibres of $T$.

We view a tensor $T \in \Rbb^{n \times n \times n}$ as a multilinear form. In particular, for vectors $u,v,w \in \R^n$, let\begin{equation} \label{eqn:rank-1 update}
 T(I,v,w) := \sum_{j,l \in [n]} v_j w_l T(:,j,l) \ \in \R^n,
\end{equation}
which is a multilinear combination of the tensor mode-$1$ fibers.
Similarly $T(u,v,w) \in \R$ is a multilinear combination of the tensor entries.

A tensor $T  \in \Rbb^{n \times n \times n}$ has a CP {\em rank} at most $r$ if it can be written as the sum of $r$ rank-$1$ tensors as
\begin{equation}\label{eqn:tensordecomp}
T = \sum_{i\in [r]} \sigma^*_i u_i \otimes u_i \otimes u_i, \quad   u_i \in \Rbb^n, \|u_i\|=1,
\end{equation}
where notation $\otimes$  represents the {\em outer product}. We sometimes abbreviate  $a \otimes a \otimes a$ as $a^{\otimes 3}$. Without loss of generality, $\sigma^*_i>0$, since $-\sigma_i^* u_i^{\otimes 3}= \sigma_i^* (-u_i)^{\otimes 3}$.

\floatname{algorithm}{Algorithm}

\begin{algorithm}[t!]
  \caption{$(\widehat{L},\ \widehat{S})=$ $\ncralgo$ $(\M, \delta, r,\beta)$: Tensor Robust PCA}
  \begin{algorithmic}[1]
    \STATE {\bf Input}: Tensor $\M\in \R^{n\times n \times n}$, convergence criterion $\delta$, target rank $r$, thresholding scale parameter $\beta$. $P_l(A)$ denote estimated    rank-$l$ approximation of tensor $A$,  and let $\sigma_l(A)$ denote the estimated $l^{\tha}$ largest eigenvalue using   Procedure~\ref{algo:sspm}. $HT_\zeta(A)$ denotes hard-thresholding, i.e. $\HT_{\zeta}(A))_{ijk}=A_{ijk}$ if $|A_{ijk}|\geq \zeta$ and $0$ otherwise.
    \STATE Set initial threshold $\zeta_0 \, \leftarrow \, \beta \sigma_1(\M)$ and estimates $ S^{(0)}=\HT_{\zeta_0}(\M-L^{(0)})$.
    \FOR{Stage $l=1$ to $r$ }
    \FOR{$t=0$ to $\tau=10 \log \left(n \beta \twonorm{\M-\S^{(0)}}/\delta\right)$}
    \STATE    $\,L^{(t+1)}=P_l(\M-S^{(t)})$.
    \STATE $ S^{(t+1)}=\HT_{\zeta}(\M-L^{(t+1)}).$
    \STATE $\zeta_{t+1}\!\! =\!\!   \beta  (\sigma_{l+1}(\M\!-\!S^{(t+1)}) +\left(\frac{1}{2}\right)^t \sigma_l(\M-S^{(t+1)})  )$.
    \ENDFOR
    \STATE If {$\beta \sigma_{l+1}(\Ltn) < \frac{\delta}{2n}$}, then  return $\L^{(\tau)},\S^{(\tau)}$, else reset $S^{(0)}= S^{(\tau)}$
    \ENDFOR
    \STATE {\bf Return: }$\widehat{L}=L^{(\tau)}, \widehat{S}=S^{(\tau)}$
  \end{algorithmic}
  \label{algo:sap}
\end{algorithm}
\vspace*{-5pt}
\paragraph{$\ncralgo$ method: }We propose  non-convex  algorithm  $\ncralgo$ for robust tensor decomposition, described in Algorithm~\ref{algo:sap}. The algorithm proceeds in stages, $l=1,\ldots, r$, where $r$ is the target rank of the low rank estimate. In $l^{\tha}$ stage, we    consider alternating steps of low rank projection $P_l(\cdot)$ and hard thresholding of the residual, $\HT(\cdot)$. For computing $P_l(\widetilde{L})$, that denotes the $l$ leading eigenpairs of $\widetilde{L}$, we execute gradient ascent on a function $f(v)=\widetilde{L}(v, v, v)-\lambda \|v\|^4$ with multiple restarts and deflation (see Procedure~\ref{algo:sspm}).

\floatname{algorithm}{Procedure}
\renewcommand{\thealgorithm}{1}
\begin{algorithm}[t!]
  \caption{$\{\hat{L}_l, (\hat{u_j}, \lambda_j)_{j\in [l]}\}=P_l(\M)$: GradAscent (Gradient Ascent method)}\label{algo:sspm}
  \begin{algorithmic}[1]
\STATE {\bf Input}: Symmetric tensor $\M\in \R^{n\times n \times n}$,   target rank $l$, exact rank $r$, $N_1$ number of initializations or restarts, $N_2$ number of power iterations for each initialization. Let $\M_1\leftarrow \M.$ 
\FOR{$j=1,\ldots, r$}
\FOR {$i = 1, \ldots, N_1$}
\STATE $\theta \sim \mathcal{N} (0, I_n)$.  Compute   top singular vector $u$ of $\M_j(I,I,\theta)$. Initialize $v_i^{(1)}\leftarrow u$. Let $\lambda=\M_j(u, u, u)$.
\REPEAT
\STATE{$v_i^{(t+1)} \leftarrow \M_j(I,v_i^{(t)},v_i^{(t)}) / \| \M_j(I,v_i^{(t)},v_i^{(t)}) \|_2$} \COMMENT {Run power method to land in spectral ball}
\STATE $\lambda_i^{(t+1)} \leftarrow \M_j(v_i^{(t+1)},v_i^{(t+1)},v_i^{(t+1)})$
\UNTIL {$t =N_2$} 
\STATE Pick the best: reset $i \leftarrow \argmax_{i \in [N_1]} T_j(v_i^{(t+1)},v_i^{(t+1)},v_i^{(t+1)})$ and $\lambda_i = \lambda_i^{(t+1)}$ and $v_i = v_i^{(t+1)}$.
\STATE Deflate: $\M_j\leftarrow \M_j - \lambda_i v_i \otimes v_i \otimes v_i $.
\ENDFOR
\ENDFOR
\FOR{$j=1,\ldots, r$}
\REPEAT
\STATE Gradient Ascent iteration: {\small $v_j^{(t+1)} \leftarrow v_j^{(t)}+\frac{1}{4 \lambda (1+\lambda/\sqrt{n})}\cdot \left(\M(I, v_j^{(t)}, v_j^{(t)})-\lambda \| v_j^{(t)}\|^2v_j^{(t)}\right)$.}
\UNTIL convergence  (linear rate, refer Lemma~\ref{lem:grad_asc}).
\STATE Set $\widehat{u}_j = v_j^{(t+1)}$, $\lambda_j  = \M (v_j^{(t+1)}, v_j^{(t+1)}, v_j^{(t+1)})$
\ENDFOR
\RETURN Estimated top $l$ out of all the top $r$ eigenpairs
$(\widehat{u}_j, \lambda_j)_{j\in [l]}$, and low rank estimate $\hat{L}_l = \sum_{i \in [l]} \lambda_i \widehat{u}_j\otimes \widehat{u}_j\otimes \widehat{u}_j$.
\end{algorithmic}
\end{algorithm}

\vspace*{-5pt}
\paragraph{Computational complexity: }In $\ncralgo$, at the $l^{\tha}$ stage, the $l$-eigenpairs are computed using Procedure~\ref{algo:sspm}, whose complexity is $O(n^3  l N_1 N_2)$.  The hard thresholding operation $\HT_{\zeta}(\M-L^{(t+1)})$ requires $O(n^3)$ time. We have $O\left( \log \left( \frac{n \|\M\|}{\delta} \right)\right)$ iterations for each stage of the $\ncralgo$ algorithm and there are $r$ stages. By Theorem~\ref{thm:robustpower}, it suffices to have $N_1= \tilde{O }\left(n^{1+c}  \right)$  and $
N_2 =\tilde{O} \left( 1\right)$,  and where $\tilde{O}(\cdot)$ represents $O(\cdot)$ up to polylogarithmic factors and  $c$ is a small constant. Hence, the overall computational complexity of $\ncralgo$ is $\tilde{O} \left(n^{4+c} r^2    \right)$.

\section{Theoretical Guarantees}
In this section, we provide guarantees for the $\ncralgo$ proposed in the previous section for RTD in \eqref{eqn:robust-def}. Even though we consider a symmetric $\Lo$ and $\So$ in \eqref{eqn:robust-def}, we can extend the results to asymmetric tensors, by embedding them into symmetric tensors, on lines of~\cite{ragnarsson2013block}.

In general,  it is impossible to have a unique recovery of  low-rank and sparse components.
Instead, we assume the following conditions to guarantee uniqueness:

\textbf{(L)} $\Lo$ is a rank-$r$ orthogonal tensor in \eqref{eqn:robust-def}, i.e., $\inner{u_i, u_j}=\delta_{i,j}$, where $\delta_{i,j}=1$ iff $i=j$ and $0$ o.w. $\Lo$ is $\mu$-incoherent, i.e.,    $\|u_i \|_\infty \leq \frac{\mu}{n^{1/2}}$ and $\sigma_i^* > 0$, $\forall 1\leq i\leq r$.

The above conditions of having an incoherent low rank tensor  $\Lo$ are similar to the conditions for robust matrix PCA. For tensors, the assumption of an orthogonal decomposition is limiting, since there exist tensors whose CP decomposition is non-orthogonal~\cite{AnandkumarEtal:tensor12}.  We later discuss how our  analysis can be extended to  non-orthogonal tensors. We now list the conditions for sparse tensor $\So$.

The tensor $\So$ is block sparse, where each block has at most $d$ non-zero entries along any fibre and the number of blocks is $B$. Let $\Psi$ be the support tensor that has the same sparsity pattern as $\So$, but with unit entries, i.e. $\Psi_{i,j,k}=1$ iff. $\So_{i,j,k}\neq 0$ for all $i,j,k\in [n]$. We now make assumptions on sparsity pattern of $\Psi$. Let $\eta$ be the maximum fraction of overlap between any two block fibres $\psi_i$ and $\psi_j$. In other words,  $\max_{i\neq j} \inner{\psi_i, \psi_j} < \eta d$.
\textbf{(S)} Let   $d$ be the sparsity level along any fibre of a block and let $B$ be  the number of blocks. We require
    \begin{align} \label{eqn:psi}\Psi= \sum_{i=1}^B   \psi_i \otimes \psi_i \otimes \psi_i, \|\psi_i\|_0 \leq d, \, \psi_i(j)=0\mbox{ or } 1 \nonumber \\
    \forall i\in[B],\,j\in [n],\\
    \label{eqn:S1} d =O ( n / r \mu^3)^{2/3},  B =O( \min ( n^{2/3} r^{1/3}, \, \eta^{-1.5} )).\end{align}
We assume a block sparsity model for $\So$ above. Under this   model, the support tensor $\Psi$ which encodes sparsity pattern, has rank $B$, but not the sparse tensor $\So$ since the entries are allowed to be arbitrary. We also note that we set $d$ to be $n^{2/3}$ for ease of exposition and show one concrete example where our method significantly outperforms matrix robust PCA methods.

As discussed in the introduction, it may be advantageous to consider tensor methods for robust decomposition only when sparsity across the different matrix slices are aligned/structured in some manner, and a block sparse model is a natural structure to consider.  We later demonstrate the precise nature of superiority of tensor methods under block sparse perturbations.

For the above mentioned sparsity structure, we set $\beta=4 \mu^3 r / n^{3/2}$  in our algorithm. Under the above conditions, our proposed algorithm $\ncralgo$ establishes convergence to the globally optimal solution.

\begin{theorem}[Convergence to global optimum for $\ncralgo$]
Let $\Lo, \So$ satisfy $(L)$ and $(S)$, and $\beta= 4 \frac{\mu^3 r}{n^{3/2}}$. The outputs $\widehat{\L}$ (and its parameters $\hat{u}_i$ and $\hat{\lambda}_i$) and $\widehat{\S}$ of Algorithm~\ref{algo:sap} satisfy w.h.p.:
\begin{align*}
  \left\| \hat{u}_i - u_i \right\|_\infty \leq \delta / \mu^2 r n^{1/2} \sigma_{\min}^*, \\
  | \hat{\lambda}_i - \sigma_i^* | \leq \delta, \ \ \ \forall i  \in [n], \\
\frob{\widehat{\L} - \Lo} \leq \delta, \ \ \infnorm{\widehat{\S} - \So} \leq \delta / n^{3/2},\\ \mbox{ and }\ \ \supp{\widehat{\S}} \subseteq \supp{\So}.
\end{align*}
\label{thm:main}
\end{theorem}
\vspace*{-9pt}
\paragraph{Comparison with matrix methods: }We now compare with the matrix methods for recovering the sparse and low rank tensor components in \eqref{eqn:robust-def}.
Robust matrix PCA can be performed either on all the   matrix slices of the input tensor $M_i:=\M(I,I, e_i)$,  for $i \in [n]$, or on the {\em flattened} tensor $\M$ (see the definition in Section~\ref{sec:algo}). We can either use convex relaxation methods~\cite{chandrasekaran2011rank,candes2011robust,hsu2011robust} or non-convex methods~\cite{netrapalli2014non} for robust matrix PCA.

Recall that $\eta$ measures the fraction of overlapping entries between any two different block fibres, i.e. $\max_{i\neq j} \inner{\psi_i, \psi_j} < \eta d$, where $\psi_i$ are the fibres of the block components of tensor $\Psi$ in \eqref{eqn:psi} which encodes the sparsity pattern of $\So$ with $0$-$1$ entries. A short proof is given in Appendix~\ref{proof:cor}.

\begin{corollary}[Superiority of tensor methods]\label{cor:sup}
The proposed tensor method $\ncralgo$ can handle perturbations $\So$ at a higher sparsity level compared to performing  matrix robust PCA on either matrix slices or the flattened tensor using guaranteed methods in~\cite{hsu2011robust,netrapalli2014non} when the (normalized) overlap between different blocks satisfies  $\eta =O(  r / n )^{2/9}$.
\end{corollary}
\vspace*{-5pt}
\paragraph{Simplifications under random block sparsity: } We now obtain transparent results for a  random block sparsity model, where the components   $\psi_i$ in \eqref{eqn:psi} for the support tensor $\Psi$ are drawn uniformly among all $d$-sparse vectors.
In this case, the incoherence $\eta$ simplifies as $\eta \overset{w.h.p}{=} O (\frac{d}{n} ) $ when  $d> \sqrt{n}$ and $\eta \overset{w.h.p}{=}O\ (1 / \sqrt{n} ),$  o.w.  Thus, the condition on  $B$ in \eqref{eqn:S1} simplifies as  $B= O(\min( n^{2/3} r^{1/3}, ( n / d)^{1.5}))$ when  $d> \sqrt{n}$ and $B= O(\min( n^{2/3} r^{1/3}, n^{0.75}))$  o.w. Recall that as before, we require sparsity level of a fibre in any block as $d =O( n /  r \mu^3)^{2/3}$.
For simplicity, we  assume that $\mu=O(1)$ for the remaining section.

We now explicitly compute the sparsity level of $\So$ allowed by our method and compare it to the level allowed by matrix based robust PCA.

\begin{corollary}[Superiority of tensor methods under random sparsity] Let $D_{\ncralgo}$ be the number of non-zeros in $\So$ ($\|\So\|_0$) allowed by $\ncralgo$ under the analysis of Theorem~\ref{thm:main} and let $D_{\mat}$ be the allowed $\|\So\|_0$  using the standard matrix robust PCA analysis. Also, let $\So$ be sampled from the random block sparse model. Then, the following holds:
\bcase{\frac{D_{\ncralgo}}{ D_{\mat}}= } \Omega\left( n^{1/6}r^{4/3}\right),&\mbox{for }$r< n^{0.25}$,\label{eqn:case5}\\ \Omega\left(n^{5/12} r^{1/3}\right) ,&o.w. \label{eqn:case6}\ecase
\end{corollary}
\vspace*{-7pt}
\paragraph{Unstructured sparse perturbations $\So$: }If we do not assume block sparsity in (S), but instead assume unstructured sparsity at level $D$, i.e. the number of non zeros along any fibre of $\So$ is at most $D$, then the matrix methods are superior. In this case, for success of $\ncralgo$, we require that $D=O\left( \frac{\sqrt{n}}{ r \mu^3}\right)$ which is worse than the requirement of matrix methods $D=O(\frac{n}{r \mu^2})$. Our analysis suggests that if there is no structure in sparse patterns, then matrix methods are superior. This is possibly due to the fact that finding a low rank tensor decomposition requires more stringent conditions on the noise level. Meanwhile, when there is no block structure, the tensor algebraic constraints do not add significantly new information. However, in the experiments, we find that our tensor method $\ncralgo$ is superior to matrix methods even in this case, in terms of both accuracy and running times.

\subsection{Analysis of Procedure~\ref{algo:sspm}}
\begin{figure*}[t]
\centering
\begin{tabular}{cccc}
\hspace*{-10pt}
\psfrag{Figure}[c]{\tiny $n = 100, \mu = 1, r = 5, det. sp.$ }
\psfrag{10}[c]{\tiny 10}
\psfrag{0.6}{\tiny 0.6}
\psfrag{0.5}{\tiny 0.5}
\psfrag{0.4}{\tiny 0.4}
\psfrag{Nonwhiten}{}
\psfrag{40}[l]{\tiny \hspace{-1em} 40}
\psfrag{30}[l]{\tiny \hspace{-1em} 30}
\psfrag{20}[l]{\tiny \hspace{-1em} 20}
\psfrag{Matrix(slice)}{}
\includegraphics[width=0.2\textwidth]{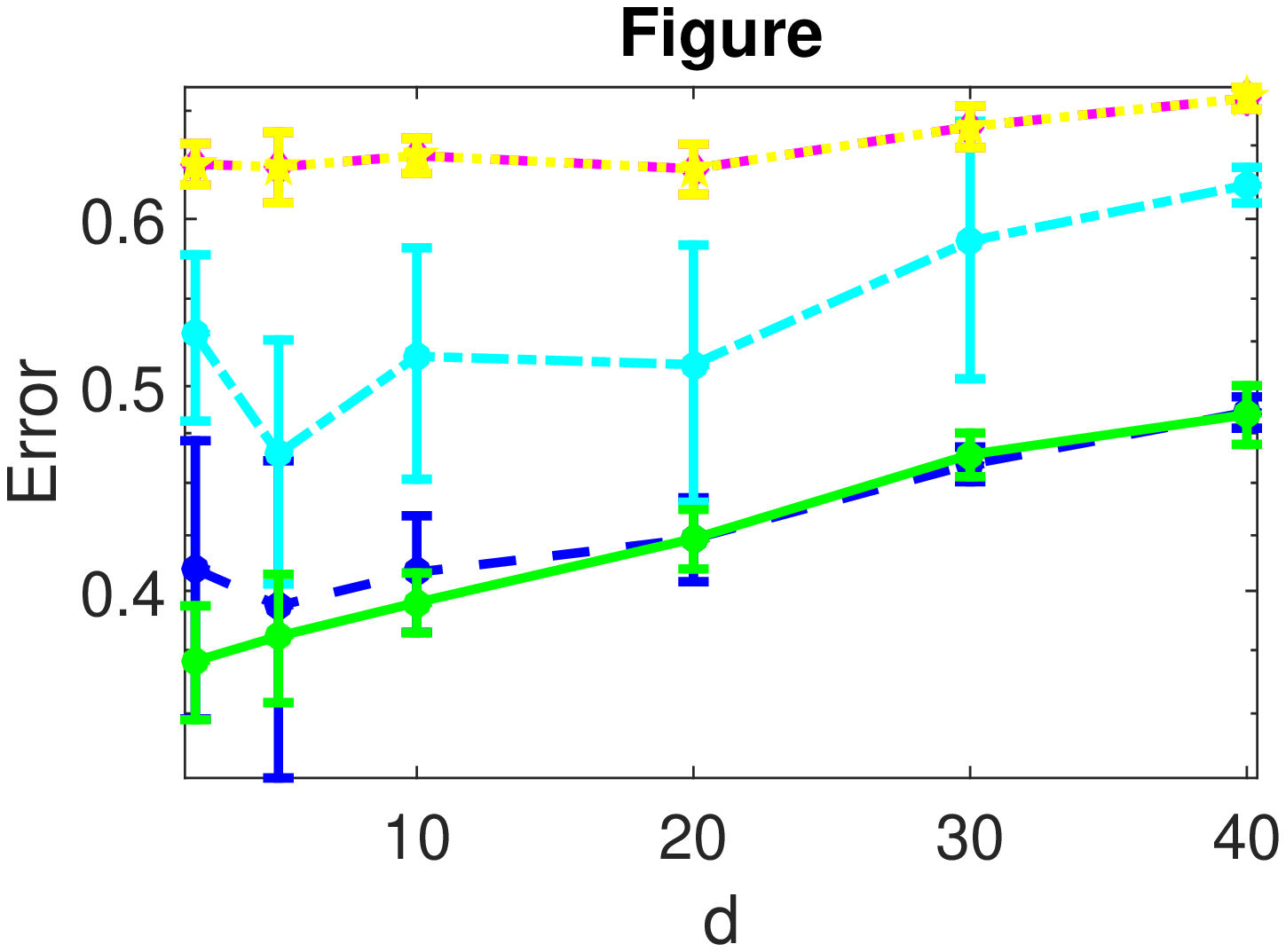}&
\hspace*{-10pt}
\psfrag{Figure}[c]{\tiny $n = 100, \mu = 1, r = 25, det. sp.$ }
\psfrag{10}[c]{\tiny 10}
\psfrag{0.8}{\tiny 0.8}
\psfrag{0.7}{\tiny 0.7}
\psfrag{0.6}{\tiny 0.6}
\psfrag{0.5}{\tiny 0.5}
\psfrag{0.4}{\tiny 0.4}
\psfrag{0.3}{\tiny 0.3}
\psfrag{Nonwhiten}{}
\psfrag{40}[l]{\tiny \hspace{-1em} 40}
\psfrag{30}[l]{\tiny \hspace{-1em} 30}
\psfrag{20}[l]{\tiny \hspace{-1em} 20}
\includegraphics[width=0.2\textwidth]{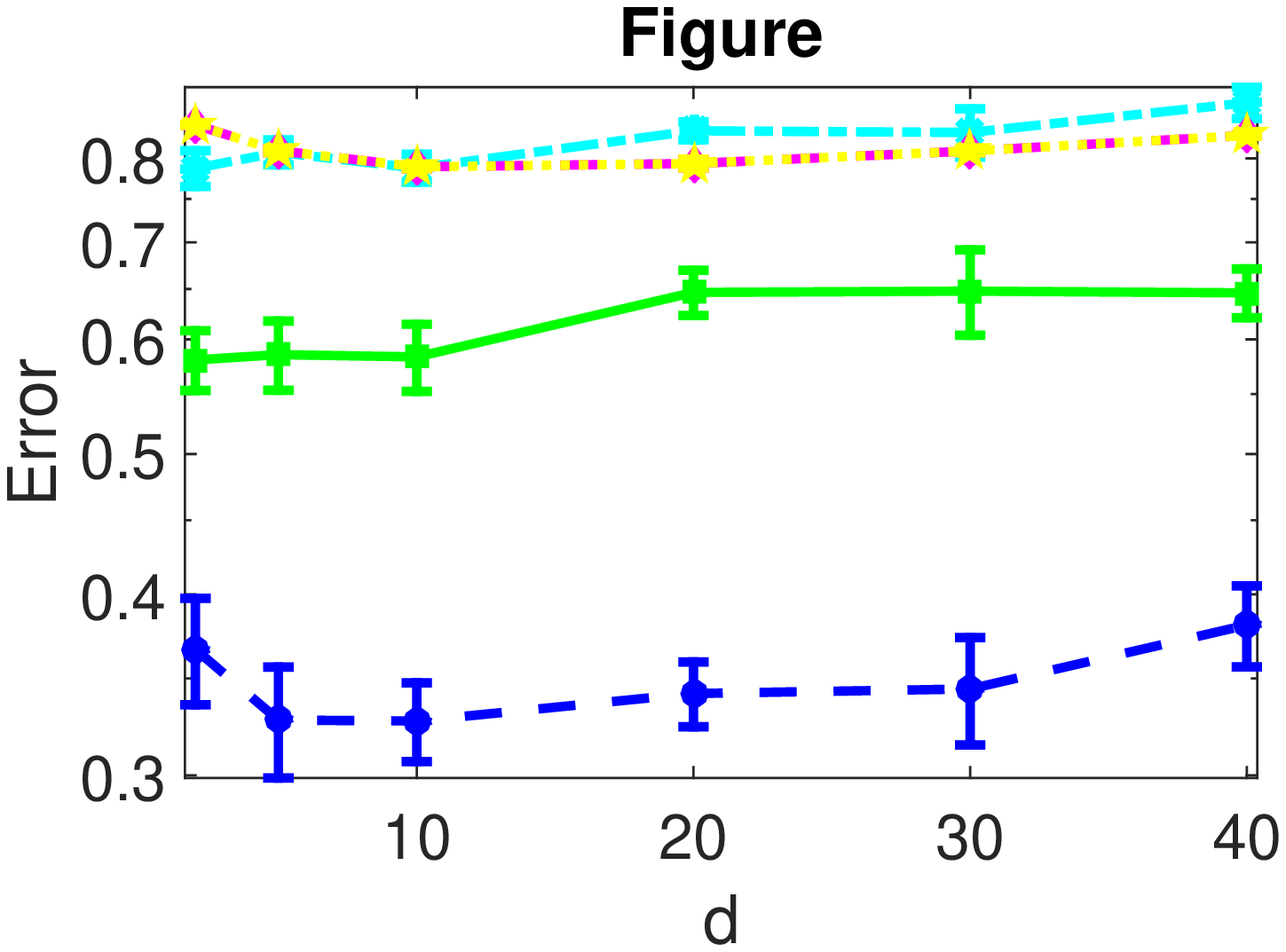}&
\hspace*{-10pt}
\psfrag{Figure}[c]{\tiny $n = 100, \mu = 1, r = 5, bl. sp.$ }
\psfrag{10}[c]{\tiny 10}
\psfrag{0}[c]{\tiny 0}
\psfrag{40}[l]{\tiny \hspace{-1em} 40}
\psfrag{30}[l]{\tiny \hspace{-1em} 30}
\psfrag{20}[l]{\tiny \hspace{-1em} 20}
\psfrag{Nonwhiten}{\tiny T-RPCA}
\psfrag{Whiten(random)}{\tiny T-RPCA-W(slice)}
\psfrag{Whiten(true)}{\tiny T-RPCA-W(true)}
\psfrag{Matrix(slice)}{\tiny M-RPCA(slice)}
\psfrag{Matrix(flat)}{\tiny M-RPCA(flat)}
\includegraphics[width=0.2\textwidth]{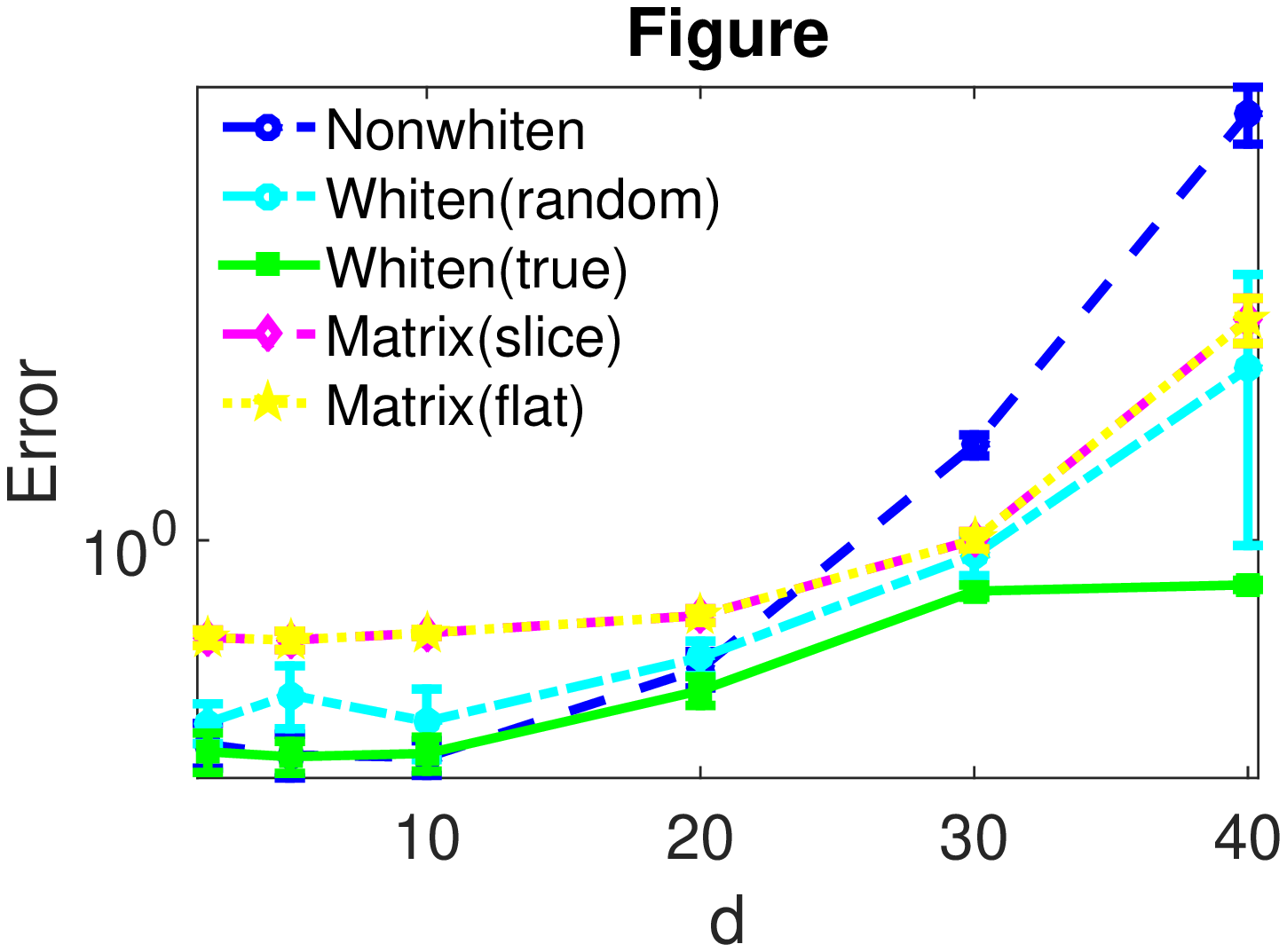}&
\hspace*{-10pt}
\psfrag{Figure}[c]{\tiny $n = 100, \mu = 1, r = 25, bl. sp.$ }
\psfrag{10}[c]{\tiny 10}
\psfrag{0}[c]{\tiny 0}
\psfrag{40}[l]{\tiny \hspace{-1em} 40}
\psfrag{30}[l]{\tiny \hspace{-1em} 30}
\psfrag{20}[l]{\tiny \hspace{-1em} 20}
\includegraphics[width=0.2\textwidth]{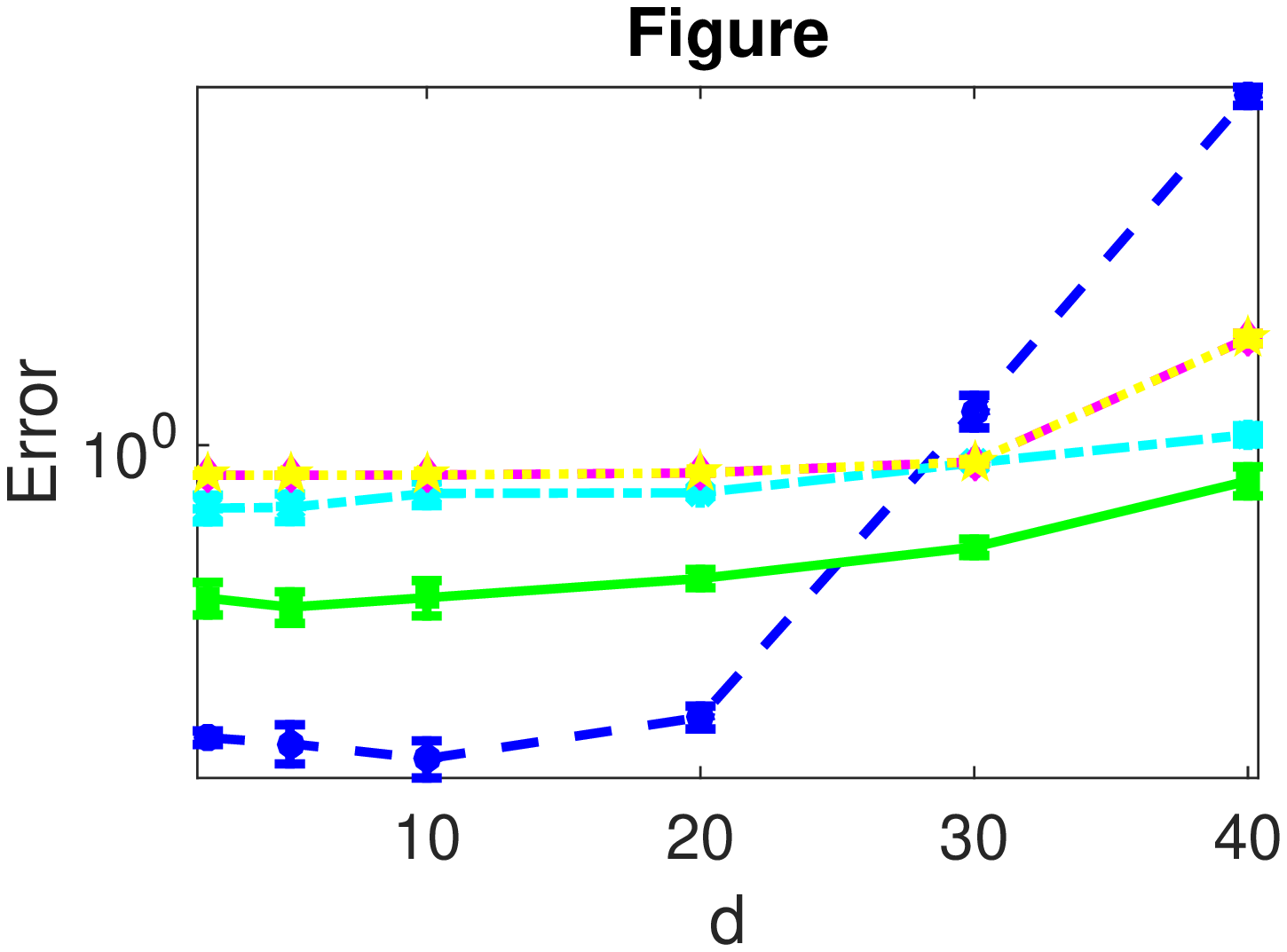}\\[-3pt]
(a)&(b)&(c)&(d)\\
\end{tabular}
\caption{\small (a) Error comparison of different methods with deterministic sparsity, rank $5$, varying $d$. (b) Error comparison of different methods with deterministic sparsity, rank $25$, varying $d$. (c) Error comparison of different methods with block sparsity, rank $5$, varying $d$. (d) Error comparison of different methods with block sparsity, rank $25$, varying $d$. Error = $\|L^*-L\|_F/\|L^*\|_F$. The curve labeled `T-RPCA-W(slice)' refers to considering recovered low rank part from a random slice of the tensor $T$ by using matrix non-convex RPCA method as the whiten matrix, `T-RPCA-W(true)' is using true second order moment in whitening, `M-RPCA(slice)' treats each slice of the input tensor as a non-convex matrix-RPCA(M-RPCA) problem, `M-RPCA(flat)' reshapes the tensor along one mode and treat the resultant as a matrix RPCA problem. All four sub-figures share same curve descriptions.}
\label{fig:synthetic_plots_1}
\end{figure*}

\begin{figure*}[t]
\centering
\begin{tabular}{cccc}
\hspace*{-10pt}
\psfrag{Figure}[c]{\tiny $n = 100, \mu = 1, r = 5, det. sp.$ }
\psfrag{0}[l]{\tiny \hspace{-1em} 0}
\psfrag{10}[l]{\tiny \hspace{-1em} 10}
\psfrag{40}[l]{\tiny \hspace{-1em} 40}
\psfrag{30}[l]{\tiny \hspace{-1em} 30}
\psfrag{20}[l]{\tiny \hspace{-1em} 20}
\includegraphics[width=0.2\textwidth]{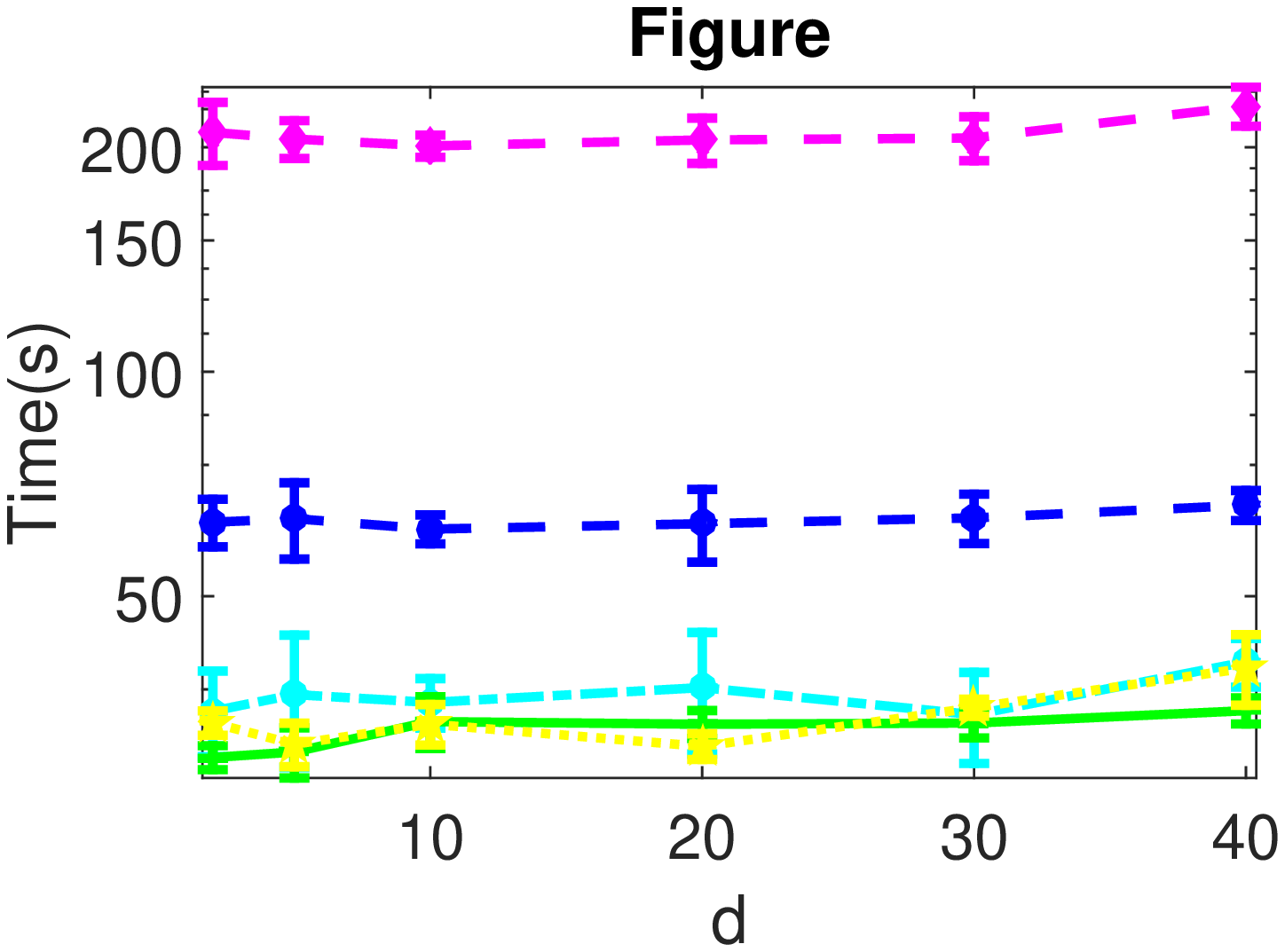}&
\hspace*{-10pt}
\psfrag{0}[l]{\tiny \hspace{-1em} 0}
\psfrag{10}[l]{\tiny \hspace{-1em} 10}
\psfrag{40}[l]{\tiny \hspace{-1em} 40}
\psfrag{30}[l]{\tiny \hspace{-1em} 30}
\psfrag{20}[l]{\tiny \hspace{-1em} 20}
\psfrag{Figure}[c]{\tiny $n = 100, \mu = 1, r = 25, det. sp.$ }
\includegraphics[width=0.2\textwidth]{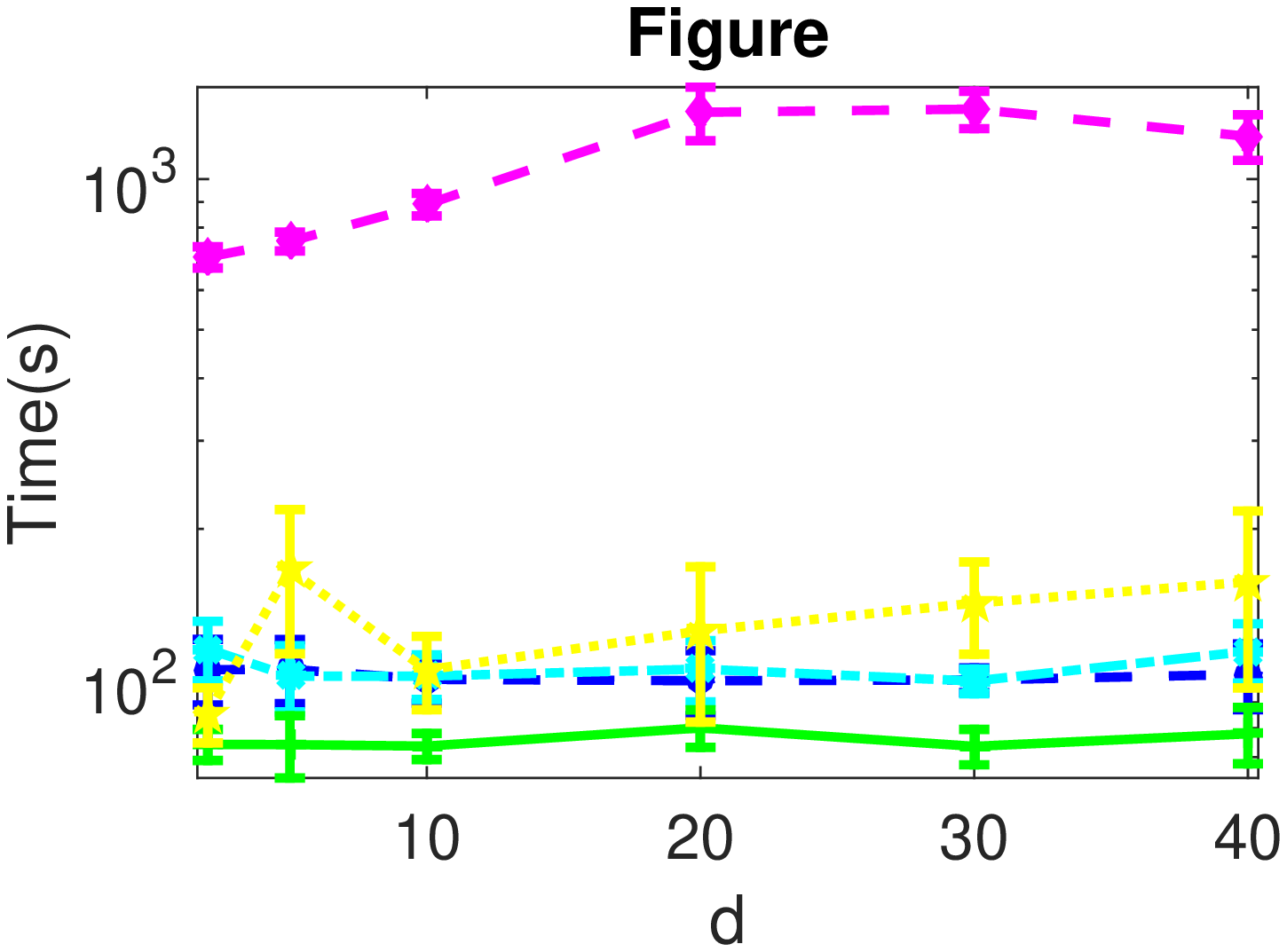}&
\hspace*{-10pt}
\psfrag{Figure}[c]{\tiny $n = 100, \mu = 1, r = 5, bl. sp.$ }
\psfrag{0}[l]{\tiny \hspace{-1em} 0}
\psfrag{10}[l]{\tiny \hspace{-1em} 10}
\psfrag{40}[l]{\tiny \hspace{-1em} 40}
\psfrag{30}[l]{\tiny \hspace{-1em} 30}
\psfrag{20}[l]{\tiny \hspace{-1em} 20}
\psfrag{Nonwhiten}{\tiny T-RPCA}
\psfrag{Whiten(random)}{\tiny T-RPCA-W(slice)}
\psfrag{Whiten(true)}{\tiny T-RPCA-W(true)}
\psfrag{Matrix(slice)}{\tiny M-RPCA(slice)}
\psfrag{Matrix(flat)}{\tiny M-RPCA(flat)}
\includegraphics[width=0.2\textwidth]{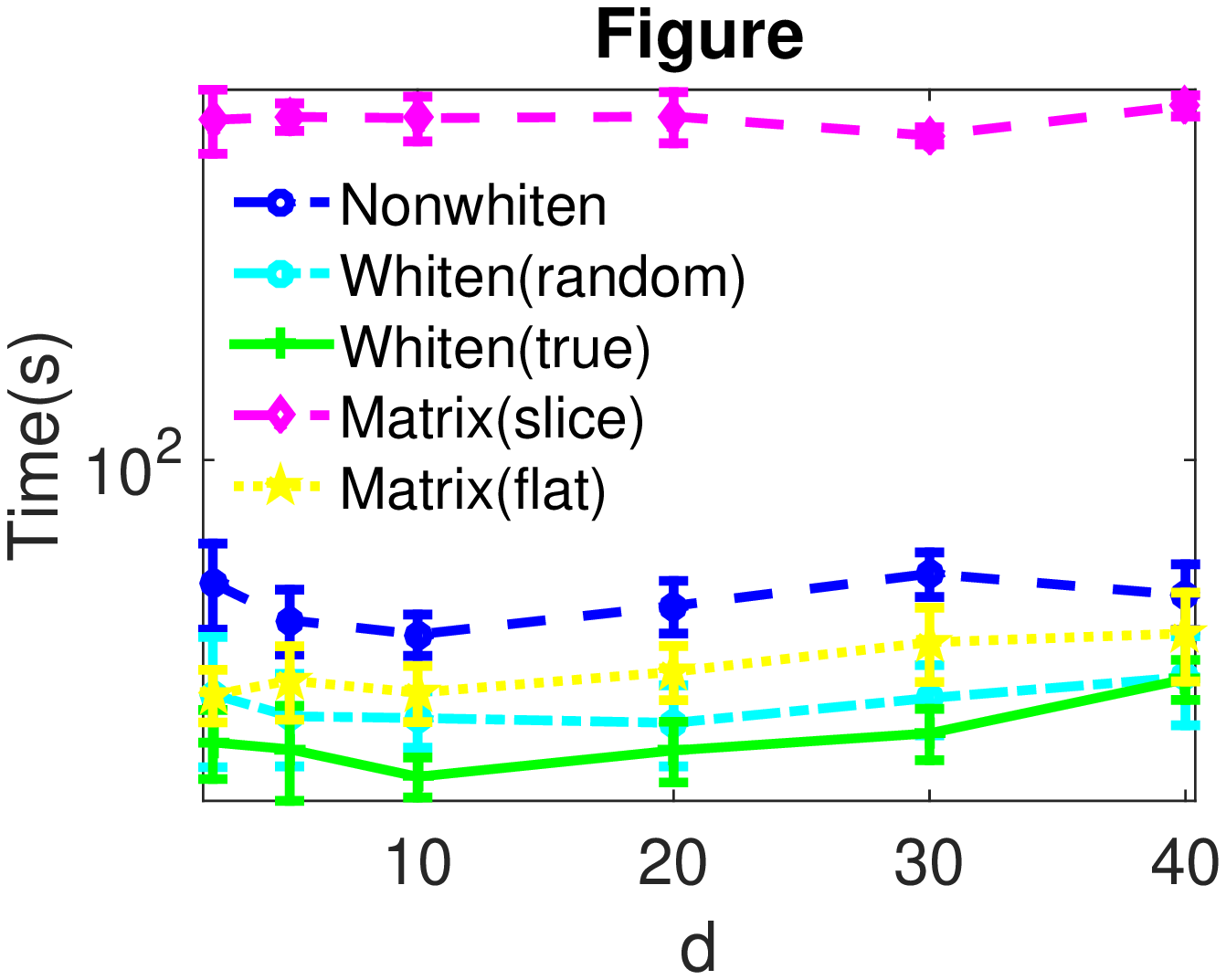}&
\hspace*{-10pt}
\psfrag{0}[l]{\tiny \hspace{-1em} 0}
\psfrag{10}[l]{\tiny \hspace{-1em} 10}
\psfrag{40}[l]{\tiny \hspace{-1em} 40}
\psfrag{30}[l]{\tiny \hspace{-1em} 30}
\psfrag{20}[l]{\tiny \hspace{-1em} 20}
\psfrag{Figure}[c]{\tiny $n = 100, \mu = 1, r = 25, bl. sp.$ }
\includegraphics[width=0.2\textwidth]{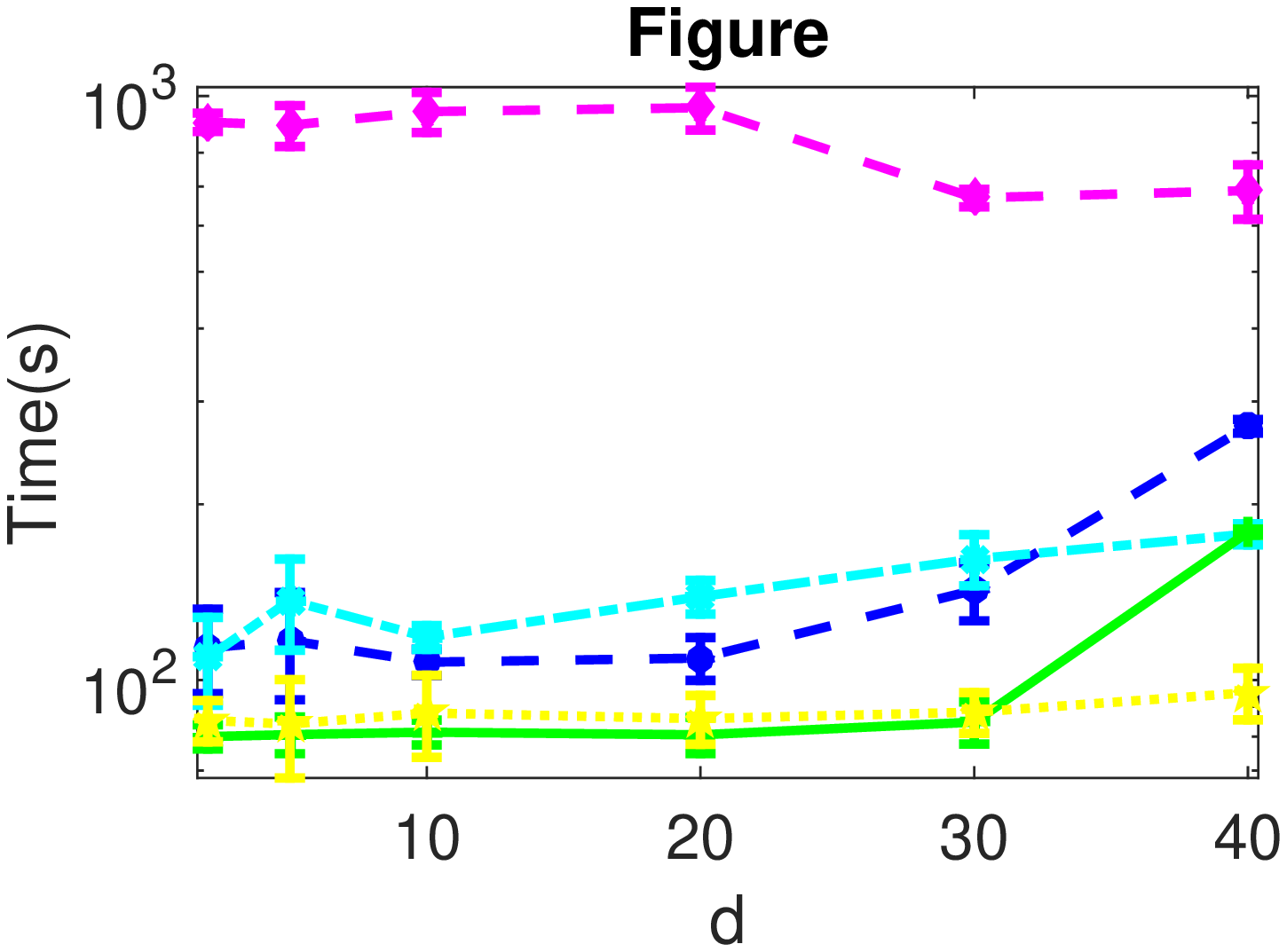}\\[-3pt]
(a)&(b)&(c)&(d)\\
\end{tabular}
\caption{\small (a) Running time comparison of different methods with deterministic sparsity, rank $5$, varying $d$. (b) Running time comparison of different methods with deterministic sparsity, rank $25$, varying $d$. (c) Running time comparison of different methods with block sparsity, rank $5$, varying $d$. (d) Running time comparison of different methods with block sparsity, rank $25$, varying $d$. Curve descriptions are same as in Figure 1.}
\label{fig:synthetic_plots_2}
\end{figure*}

Our proof of Theorem~\ref{thm:main} depends critically on an assumption that Procedure~\ref{algo:sspm} indeed obtains the top-$r$ eigen-pairs. We now concretely prove this claim. Let $\widetilde{L}$ be a symmetric tensor which is a perturbed version of an orthogonal tensor $\Lo$, $ \widetilde{L} = \Lo+ E \in \R^{n \times n \times n}, \quad \Lo = \sum_{i\in [r]} \sigma^*_i
u_i^{\otimes 3},$  where  $\sigma^*_1 \geq \sigma^*_2\ldots \sigma^*_r > 0$ and $\{ u_1, u_2, \dotsc, u_r \}$
form an orthonormal basis.

Recall that $N_1$ is the number of initializations to seed the power method,   $N_2$ is the number of  power iterations,   and $\delta$ is the convergence criterion.
For any $\xi \in (0,1)$, and  $l \leq r$, assume the following
\begin{align*}\label{eqn:condpower}
&\|E\| \leq O ( \sigma^*_l / \sqrt{n} ) ,
\, N_1= O (n^{1+c} \log (1 / \xi ) ), \\
&N_2 \geq \Omega    (   \log\left( {k} \right) +  \log \log ( \sigma^*_{\max}/\|E\| )   ),
\end{align*} where $c$ is a small constant.
We now state the main result for recovery of components of $\Lo$ when Procedure~\ref{algo:sspm} is applied to $\widetilde{L}$.

\begin{theorem}[Gradient Ascent method]
\label{thm:robustpower}
Let $\widetilde{L}=\Lo+E$ be as defined above with $\|E\|\leq O(\sigma_r^*/\sqrt{n})$.  Then, applying  Procedure~\ref{algo:sspm} with deflations on $\widetilde{L}$  with target rank $l \leq r$,  yields  $l$ eigen-pairs of $\widetilde{L}$,    given by
 $(\lambda_1, \hat{u}_1), (\lambda_2, \hat{u}_2), \dotsc,
(\lambda_l, \hat{u}_l),$ up to arbitrary small error $\delta>0$ and  with probability at least $1-\xi$. Moreover, there exists a permutation $\pi$ on
$[l]$ such that: $\forall j \in [l]$,
\begin{align*}|\sigma^*_{\pi(j)}-\lambda_j| \leq O\left(\|E\|+\delta\right),\\
\|u_{\pi(j)}-\hat{u}_j\| \leq O ( (\|E\| / \sigma^*_{\pi(j)})+\delta ).
\end{align*}
\end{theorem}

While~\cite[Thm. 5.1]{AnandkumarEtal:tensor12} considers  power method, here we consider the power method followed by a gradient ascent procedure. With both  methods, we   obtain outputs $(\lambda_i, \hat{u}_i)$ which are ``close'' to the original eigen-pairs of $(\sigma^*_i,u_i)$ of $\Lo$. However, the crucial difference is that  Procedure~\ref{algo:sspm}  outputs $(\lambda_i, \hat{u}_i)$   correspond to specific eigen-pairs of input tensor $\tilde{L}$, while the outputs of the usual power method have no such property and only guarantees accuracy upto $O(\|E\|_2)$ error. We critically require the eigen property of the outputs in order to guarantee contraction of error in $\ncralgo$ between alternating steps of low rank decomposition and thresholding.

The analysis of Procedure~\ref{algo:sspm} has two phases. In the first phase, we prove  that with $N_1$ initializations and $N_2$ power iterations, we get close to true eigenpairs of $\Lo$, i.e. $(\sigma^*_i,u_i)$ for $i \in [l]$. After this, in the second phase, we prove  convergence to   eigenpairs of $\widetilde{L}$.

The proof for the first phase  is on lines of proof in~\cite{AnandkumarEtal:tensor12}, but with improved requirement on error tensor $E$ in~\eqref{eqn:condpower}. This is due to the use of SVD initializations rather than random initializations to seed the power method, and its analysis is given in~\cite{DBLP:journals/corr/AnandkumarGJ14}.

Proof of the second phase follows using two observations: a) Procedure~\ref{algo:sspm} is just a simple gradient ascent of the following program: $  f(v)=\widetilde{L}(v, v, v)-\frac{3}{4}\lambda \|v\|_2^4$, b) with-in a small distance to the eigenvectors of $\widetilde{L}$, $f(v)$ is   strongly concave and as well as strongly smooth with appropriate parameters. See below lemma for a detailed proof of the above claim. Hence, using our initialization guarantee from the phase-one, Procedure~\ref{algo:sspm} converges to a $\delta$ approximation to eigen-pair of $\widetilde{L}$ in time $O(\log(1/\delta))$ and hence, Theorem~\ref{thm:robustpower} holds.
\vspace*{-5pt}
\begin{lemma}
  Let $f(v)=\widetilde{L}(v, v, v)-\frac{3}{4}\lambda \|v\|_2^4$. Then, $f$ is $\sigma_i^*(1-\frac{300\sigma_r^*}{\sqrt{n}})$-strongly concave and $\sigma_i^*(1+\frac{300\sigma_r^*}{\sqrt{n}})$ strongly smooth at all points $(v, \lambda)$ s.t. $\|v-u_i\|\leq \frac{10}{\sqrt{n}}$ and $|\lambda-\sigma_i^*|\leq \frac{10\sigma_r^*}{\sqrt{n}}$, for some $1\leq i\leq r$. Procedure~\ref{algo:sspm} converges to an eigenvector of $\widetilde{L}$ at a linear rate.
  \label{lem:grad_asc}
\end{lemma}
\vspace*{-5pt}
\begin{proof}
  Consider the gradient and Hessian of $f$ w.r.t. $v$:
  \begin{align}
    \nabla f&=3\widetilde{L}(I, v, v) - 3 \lambda \|v\|^2 v, \\
H&=6 \widetilde{L}(I, I, v) - 6\lambda vv^\top - 3\lambda \|v\|^2 I.
  \end{align}
We first note that the stationary points of $f$ indeed correspond to eigenvectors of $\widetilde{L}$ with eigenvalues $\lambda \|v\|^2$. Moreover, when $|\lambda - \sigma_i^*|\leq \frac{10\sigma_r^*}{\sqrt{n}}$ and $\|v-u_i\|\leq \frac{10}{\sqrt{n}}$, we have:
$$\|H- (-3\sigma^*_i I)\|_2 \leq 30\frac{\sigma_r^*}{\sqrt{n}}+180\frac{\sigma_r^*}{\sqrt{n}}.$$
Recall that $\widetilde{L}=L^*+E$, where $L^*$ is an orthogonal tensor and $\|E\|_2\leq \sigma_r^*/\sqrt{n}$. Hence, there exists one eigenvector in each of the above mentioned set, i.e., set of $(v, \lambda)$ s.t. $|\lambda - \sigma_i^*|\leq \frac{10\sigma_r^*}{\sqrt{n}}$ and $\|v-u_i\|\leq \frac{10}{\sqrt{n}}$. Hence, the standard gradient ascent procedure on $f$ would lead to convergence to an eigenvector of $\widetilde{L}$.
\end{proof}
\vspace*{-7pt}
\paragraph{Extending to non-orthogonal low rank tensors: }In (L), we assume that the low rank tensor $\Lo$ in \eqref{eqn:robust-def} is orthogonal. We can also  extend to non-orthogonal tensors $\Lo$, whose components $u_i$ are linearly independent. Here, there exists an invertible transformation  $W$ known as {\em whitening} that orthogonalizes the tensors~\cite{AnandkumarEtal:tensor12}.  We can incorporate whitening in   Procedure~\ref{algo:sspm} to find low rank tensor decomposition, within the iterations of $\ncralgo$. 

\section{Experiments}
\begin{figure*}[t]
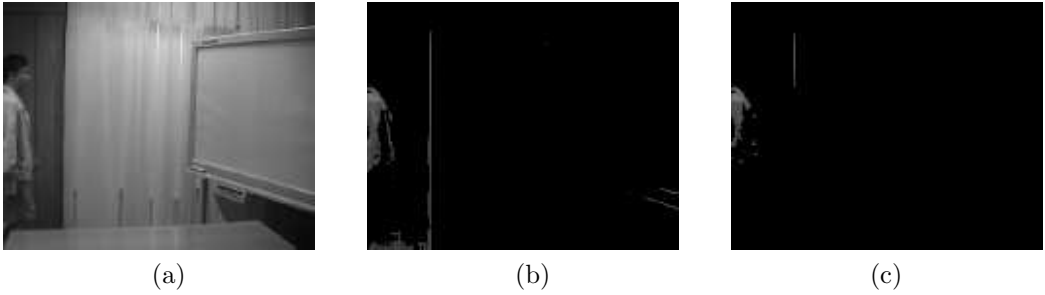

\centering
\begin{tabular}{ccc}
\includegraphics[width=0.25\textwidth]{new_algo_Tvid_m128n160r_hat1_orig.eps} \hspace{5pt} &
\includegraphics[width=0.25\textwidth]{new_algo_Tvid_m128n160r_hat1_S_t_ten.eps} \hspace{5pt} &
\includegraphics[width=0.25\textwidth]{new_algo_Tvid_m128n160r_hat1_S_t_mat.eps} \\
(a)&(b)&(c)
\end{tabular}
\caption{Foreground filtering or activity detection in the {\em Curtain} video dataset. (a): Original image frame. (b): Foreground filtered (sparse part estimated) using tensor method; time taken is $5.1s$. (c): Foreground filtered (sparse part estimated) using matrix method; time taken is $5.7s$.}
\label{fig:curtain}
\end{figure*}

We now present an empirical study of our method. The goal of this study is three-fold: a) establish that our method indeed recovers the low-rank and sparse parts correctly, without significant parameter tuning, b) demonstrate that whitening during low rank decomposition  gives computational advantages, c)  show that  our tensor methods are superior to matrix based RPCA methods in practice.

Our pseudo-code (Algorithm~\ref{algo:sap}) prescribes the threshold $\zeta$ in Step 5, which depends on the knowledge of the singular values of the low rank component $\Lo$. Instead, in the experiments, we set the threshold at the $(t+1)$ step of $l^{\tha}$ stage as $\zeta=\mu\sigma_{l+1}(\M-\St) / n^{3/2}$.  We found that the above thresholding, in the tensor case as well, provides exact recovery while speeding up the computation significantly.
\vspace*{-5pt}
\paragraph{Synthetic datasets: }The low-rank part $\Lo=\sum_{i \in [k]} \lambda_i u_i^{\otimes 3}$ is generated from a factor matrix $U\in \R^{n\times k}$ whose entries are i.i.d. samples from $\mathcal{N}(0,1)$. For deterministic sparsity setting, $\supp(\So)$ is generated by setting each entry of $[n]\times [n]\times [n]$ to be non-zeros with probability $d/n$ and each non-zero value $\So_{ijk}$ is drawn i.i.d. from the uniform distribution over $[ r / (2 n^{3/2}), r / n^{3/2} ]$. For block sparsity setting, we randomly select $B$ numbers of $[n] \times [1]$ vectors $v_i, i = 1...B$ in which each entry is chosen to be non-zero with probability $d/n$. The value of non-zero entry is assigned similar to the one in deterministic sparsity case. Each of this vector will form a subtensor($v_i^{\otimes 3}$) and those subtensors form the whole $S$. For increasing incoherence of $\Lo$, we randomly zero-out rows of $U$ and then re-normalize them. For the CP-decomposition, we use the alternating least squares (ALS) method available in the tensor toolbox~\cite{TTB_Software}. Note that 
we use the ALS procedure in practice since we found that empirically, ALS performs quite well and is convenient to use. For whitening, we use two different whitening matrices: a) the true second order moment from the true low-rank part, b) the recovered low rank part from a random slice of the tensor $T$ by using matrix non-convex RPCA method. We compare our $\ncralgo$ with matrix non-convex RPCA in two ways: a) treat each slice of the input tensor as a  matrix RPCA problem, b) reshape the tensor along one mode and treat the resultant as a matrix RPCA problem.

We vary sparsity of $\So$ and rank of $\Lo$ for $\ncralgo$ with a fixed tensor size. We investigate performance of our method, both with and without whitening, and compare with the state-of-the-art matrix non-convex RPCA algorithm. Our results for synthetic datasets is averaged over $5$ runs. In Figure~\ref{fig:synthetic_plots_1}, we report relative error ($\|L^*-\L\|_F/\|L^*\|_F$)  by each of the methods allowing maximum number of iterations up to $100$. Comparing (a) and (b) in Figure~\ref{fig:synthetic_plots_1}, we can see that with block sparsity, $\ncralgo$ is better than matrix based non-convex RPCA method when $d$ is less than $20$. If we use whitening, the advantage of $\ncralgo$ is more significant. But when rank becomes higher, the whitening method is not helpful. In Figure~\ref{fig:synthetic_plots_2}, we illustrate the computational time of each methods. We can see that  whitening simplifies the problem and give us computational advantage. Besides, the running time for the one using tensor method is similar to the one using matrix method when we reshape the tensor as one matrix. Doing matrix slices increases the running time.
\vspace*{-10pt}
\paragraph{Real-world dataset: }To demonstrate the advantage of our method, we apply our method to  \textit{activity detection} or \textit{foreground filtering}~\cite{li2004statistical}. The goal of this task is to detect activities from a video coverage, which is a set of image frames that form a tensor. In our robust decomposition framework, the moving objects correspond sparse (foreground) perturbations which we wish to filter out. The static ambient background is of lesser interest since nothing changes.

We selected one of datasets, namely the \textit{Curtain} video dataset wherein a person walks in and out of the room between the frame numbers $23731$ and $23963$. We compare our tensor method with the state-of-the-art matrix method in \cite{netrapalli2014non} on the set of $233$ frames where the activity happens. In our tensor method, we   preserve the multi-modal nature of videos and consider the set of image frames without vectorizing them. For the matrix method, we follow the setup of \cite{netrapalli2014non} by reshaping each image frame into a vector and stacking them together. We set the convergence criterion to $10^{-3}$ and run both the methods. Our tensor method yields a $10$\% speedup and obtains a better visual recovery for the same convergence accuracy as shown in Figure~\ref{fig:curtain}. 

\section{Conclusion}
We proposed a non-convex alternating method for decomposing a tensor into low rank and sparse parts. We established convergence to the globally optimal solution under natural conditions such as incoherence of the low rank part and bounded  sparsity levels for the sparse part.  We prove that our proposed tensor method can handle perturbations at a much higher sparsity level compared to robust matrix methods.  Our simulations show superior performance of our tensor method,  both in terms of accuracy and computational time. Some future directions are analyzing: (1) our method with whitening (2) the setting where grossly corrupted samples arrive in streaming manner.

\subsubsection*{Acknowledgements}
Animashree Anandkumar is supported in part by Microsoft Faculty Fellowship, NSF Career Award CCF-1254106,  ONR Award N00014-14-1-0665, ARO YIP Award W911NF-13-1-0084, and AFOSR YIP FA9550-15-1-0221. Yang Shi is supported by NSF Career Award CCF-1254106 and ONR Award N00014-15-1-2737, Niranjan is supported by NSF BigData Award IIS-1251267 and ONR Award N00014-15-1-2737.

\newpage
\appendix

\newpage

\appendix
\onecolumn
\section{Bounds for block sparse tensors}
One of the main bounds to control is the spectral norm of the sparse perturbation tensor $S$. The success of the power iterations and the improvement in accuracy  of recovery over iterative steps of $\ncralgo$ requires this bound.

\begin{lemma}[Spectral norm bounds for block sparse tensors]
\label{lem:sparse2}
Let $M \in \R^{n \times n \times n}$ satisfy the block sparsity assumption (S). Then \beq\label{lem:sparse_1} \| M \|_2=O( d^{1.5} \|M\|_\infty).\eeq\end{lemma}
\begin{proof} Let $\Psi\in \R^{n \times n \times n}$  be a tensor that encodes the sparsity of $M$ i.e.   $\Psi_{i,j,k}=1$ iff $\So_{i,j,k}\neq 0$ for all $i,j,k\in [n]$. We have that \begin{align*} \|M \|&= \max_{u: \|u\|=1} \sum_{i,j,k} M_{i,j,k} u(i) u(j) u(k) \\
& = \max_{u: \|u\|=1} \sum_{i,j,k} M_{i,j,k} \Psi_{i,j,k} u(i) u(j) u(k)\\ &\leq \max_{u: \|u\|=1} \sum_{i,j,k}| M_{i,j,k} \Psi_{i,j,k} u(i) u(j) u(k)|\\ &\leq \|M \|_\infty \max_{u: \|u\|=1} \sum_{i,j,k}  | \Psi_{i,j,k} u(i) u(j) u(k)|= \|M \|_\infty \|\Psi\|,\end{align*} where the last inequality is from Perron Frobenius theorem for non-negative tensors~\cite{chang2008perron}. Note that $\Psi$ is non-negative by definition. Now we bound $\| \Psi\|$ on lines of \cite[Lemma 4]{DBLP:journals/corr/AnandkumarGJ14}. Recall that $\forall i\in[B],\,j\in [n]$,
\[ \Psi= \sum_{i=1}^B   \psi_i \otimes \psi_i \otimes \psi_i, \quad \|\psi_i\|_0 \leq d, \, \psi_i(j)=0\mbox{ or } 1.\]
By definition $\|\psi_i \|_2 =\sqrt{d}$. Define normalized vectors $\tilde{\psi_i}:=\psi_i / \|\psi_i\|$. We have
\[ \Psi= d^{1.5}\sum_{i=1}^B   \tilde{\psi_i} \otimes \tilde{\psi_i} \otimes \tilde{\psi_i}\]
Define  matrix $\tilde{\psi}:=[\tilde{\psi_1}| \tilde{\psi_2}, \ldots \tilde{\psi_B}]$. Note that $\tilde{\psi}^\top \tilde{\psi} \in \R^{B\times B}$ is a matrix with unit diagonal entries and absolute values of off-diagonal entries bounded by $\eta$, by assumption. From Gershgorin Disk Theorem,  every subset of $L$ columns
in $\tilde{\psi}$ has singular values within $1 \pm o(1) $, where $L<\frac{1}{\eta}$.
Moreover, from Gershgorin Disk Theorem, $\|\tilde{\psi}\|< \sqrt{1+ B\eta}$.

For any unit vector $u$, let $S$ be the set of $L$ indices that are largest in $\tilde{\psi}^\top u$. By the argument above we know $\|(\tilde{\psi}_S)^\top u\| \le \|\tilde{\psi}_S\|\|u\| \le 1+o(1)$. In particular, the smallest entry in $\tilde{\psi}_S^\top u$ is at most $2/\sqrt{L}$. By construction of $S$ this implies for all $i$ not in $S$, $|\tilde{\psi}_i^\top u|$ is at most $2/\sqrt{L}$. Now we can write the $\ell_3$   norm of $\tilde{\psi}^\top u$ as
\begin{align*}
\|\tilde{\psi}^\top u\|_3^3 & = \sum_{i\in S} |\tilde{\psi}_i^\top u|^3+\sum_{i\not\in S}|\tilde{\psi}_i^\top u|^3\\
& \le \sum_{i\in S} |\tilde{\psi}_i^\top u|^2 + (2/\sqrt{L})^{3-2} \sum_{i\not\in S}|\tilde{\psi}_i^\top u|^2\\
& \le 1+2\sqrt{\eta} \|\tilde{\psi}\|^2 \leq 1+ 2 B\eta^{1.5}.
\end{align*}
Here the first inequality uses that every entry outside $S$ is small, and last inequality uses the bound argued on $\|(\tilde{\psi}_S)^\top u\|$, the spectral norm bound is assumed on $A_{S^c}$. Since $B=O(\eta^{-1.5})$, we have the result.
\end{proof}

Another important bound required is $\infty$-norm of certain contractions of the (normalized) sparse tensor and its powers, which we denote by $M$ below. We use a loose bound based on spectral norm and we require $\|M\|< 1/\sqrt{n}$.  However, this constraint will also be needed for the power iterations to succeed and is not an additional requirement. Thus, the loose bound below will suffice for our results to hold.

\begin{lemma}[Infinity norm bounds]
\label{lem:sparse3}
Let $M \in \R^{n \times n \times n}$ satisfy the block sparsity assumption (S). Let $u,v$ satisfy the assumption $\| u \|_\infty, \| v\|_\infty  \leq \frac{\mu}{n^{1/2}}$. 
Then, we have
\begin{enumerate}
\item \label{lem:sparse_4} $\| M (u, v, I) \|_\infty \leq   \frac{\kappa\mu}{n^{1/2}}\|M\|_\infty$, where $\kappa:=\frac{Bd^2  \mu}{\sqrt{n}}.$
\item \label{lem:sparse_5} $ \| [M (u,v, I)]^p \|_\infty \leq \kappa \mu \|M\|_\infty \|M\|^{p-1}$ for $p >1$.
\item \label{lem:sparse_6} $ \sum_{p \geq 1} \| [M (u, I, I)]^p v \|_\infty \leq \frac{\kappa\mu}{\sqrt{n}} \|M\|_\infty \cdot   \frac{\|M\|}{1-\|M\|}$ when $\|M\|<1/\sqrt{n}$.
\end{enumerate}
\end{lemma}
\begin{proof}
We have from norm conversion
\begin{align}\|M(u,v, I)\|_\infty
&\leq \|u\|_\infty\cdot\|v\|_\infty   \max_j \|M(I, I, e_j)\|_1\\ & \leq \frac{\mu^2}{n} \cdot   Bd^2\|M\|_\infty,  \end{align} where  $\ell_1$ norm (i.e. sum of absolute values of entries) of a slice $M(I, I, e_j)$ is $Bd^2$, since the number of non-zero entries in one block in a slice is $d^2$.

Let $Z=M(u, I, I)\in \R^{n\times n}$. Now, $\|M(u, I, I)^pv\|_\infty=\|Z^p v\|_\infty=\|Z^{p-1}a\|_\infty$ where $a=Zv$. Now, $$\|Z^{p-1}a\|_\infty=\max_j |e_j^T Z^{p-1} a|\leq \|Z^{p-1}\|_2 \|a\|_2\leq \|Z\|_2^{p-1}\|a\|_2\leq \|M\|^{p-1}\sqrt{n}\|a\|_\infty\leq \kappa \mu \|M\|_\infty \|M\|^{p-1}.$$

Hence, $ \sum_{p \geq 1} \| [M (u, I, I)]^p v \|_\infty \leq \kappa\mu \|M\|_\infty \cdot   \frac{\|M\|_2}{1-\|M\|_2}$.
\end{proof}

\section{Proof of Theorem~\ref{thm:main}}
\label{sec:rank_r}
\begin{lemma}\label{lem:threshold_est}
Let $\Lo, \So$ be symmetric and satisfy the assumptions of Theorem~\ref{thm:main} and let $\St$ be the
$t^{\textrm{th}}$ iterate of the $l^{\textrm{th}}$ stage of Algorithm~\ref{algo:sap}. Let $\sigma_1^*, \dots,
\sigma_r^*$ be the eigenvalues of $\Lo$, such that $\sigma_1^*\geq \dots \geq \sigma_r^* \geq 0$ and $\lambda_1,
\cdots,\lambda_r$ be the eigenvalues of $\M-\St$ such that $\lambda_1 \geq \cdots \geq \lambda_r \geq 0$.
Recall that $\Et \defas \So-\St$. Suppose further that
\begin{enumerate}
  \item	$\infnorm{\Et} \leq \frac{8 \mu^3 k}{n^{3/2}}\left(\sigma_{l+1}^*+\left(\frac{1}{2}\right)^{t-1}\sigma_l^*\right)$, and
  \item	$\supp{\Et} \subseteq \supp{\So}$.
\end{enumerate}
Then, for some constant $c \in [0, 1)$, we have
\begin{align}
  (1-c) \left(\sigma_{l+1}^*+\left(\frac{1}{2}\right)^{t}\sigma_l^*\right)
  \leq \left(\lambda_{l+1}+\left(\frac{1}{2}\right)^{t}\lambda_l\right)
  \leq (1+c) \left(\sigma_{l+1}^*+\left(\frac{1}{2}\right)^{t}\sigma_l^*\right).
\end{align}
\end{lemma}
\begin{proof}
Note that $\M-\St=\Lo+\Et$. Now, 
\begin{align*}
\abs{\lambda_{l+1}-\sigma_{l+1}^*} \leq 8 \twonorm{\Et} \leq 8 d^{3/2} \infnorm{\Et} \leq \frac{8\mu^3 r \gamma_t}{n^{3/2}} d^{3/2},
\end{align*}
where $\gamma_t \defas \left(\sigma_{l+1}^*+\left(\frac{1}{2}\right)^{t-1}\sigma_l^*\right)$.
That is,
$\abs{\lambda_{l+1}-\sigma_{l+1}^*}\leq  8\mu^3 r \left( \frac{d}{n} \right)^{3/2} \gamma_t$.
Similarly, $\abs{\lambda_{l}-\sigma_{l}^*} \leq  8\mu^3 r \left( \frac{d}{n} \right)^{3/2} \gamma_t$.
So we have:
\begin{align*}
\abs{ \left(\lambda_{l+1}+\left(\frac{1}{2}\right)^{t}\lambda_l\right) - \left(\sigma_{l+1}^*+\left(\frac{1}{2}\right)^{t}\sigma_l^*\right)} &\leq 8\mu^3 r \left( \frac{d}{n} \right)^{3/2} \gamma_t \left(1+\left(\frac{1}{2}\right)^t\right) \\
&\leq 16 \mu^3 r \left( \frac{d}{n} \right)^{3/2} \gamma_t \\
&\leq c \left(\sigma_{l+1}^*+\left(\frac{1}{2}\right)^{t}\sigma_l^*\right),
\end{align*}
where the last inequality follows from the bound $d \leq \left( \frac{n}{c' \mu^3 k} \right)^{2/3}$ for some constant $c'$.
\end{proof}

\begin{lemma}\label{lem:dec1}
Assume the notation of Lemma~\ref{lem:threshold_est}. Also, let $\Lt, \St$ be the $t^{\textrm{th}}$ iterates of
$r^{\textrm{th}}$ stage of Algorithm~\ref{algo:sap} and $\Ltn, \Stn$ be the $(t+1)^{\textrm{th}}$ iterates of the same
stage. Also, recall that $\Et \defas \So-\St$ and $\Etn \defas \So-\Stn$.

Suppose further that
\begin{enumerate}
  \item	$\infnorm{\Et} \leq \frac{8\mu^3 r}{n^{3/2}}\left(\sigma_{l+1}^*+\left(\frac{1}{2}\right)^{t-1}\sigma_l^*\right)$, and
  \item	$\supp{\Et}\subseteq \supp{\So}$.
\item $\|E^{(t)}\|_2 <\frac{C\sigma_l^*}{\sqrt{n}},$ where $C <1/2$ is a sufficiently small constant.
\end{enumerate}
Then, we have:
  $$\infnorm{\Ltn-\Lo} \leq  2 \frac{\mu^3 r}{n^{3/2}}\left(\sigma_{l+1}^*+\left(\frac{1}{2}\right)^{t}\sigma_l^*\right)$$
\end{lemma}
\begin{proof}
Let $\Ltn = \sum_{i=1}^{l} \lambda_i u_i^{(t+1)}$ be the eigen decomposition obtained using the tensor power method on $(\M - \St)$ at the $(t+1)^{th}$ step of the $l^{th}$ stage. Also, recall that $\M-\St=\Lo+\Et$ where $\Lo = \sum_{j=1}^{r} \sigma_j^* u_j^{\otimes 3}$. Define $\Et := \So - \St$. Define $E^i := \Et(u_i^{(t+1)}, I, I)$.  Let $\|\Et\|_2 :=\epsilon$.

Consider the eigenvalue equation $(\M - \St) (u_i^{(t+1)}, u_i^{(t+1)}, I) = \lambda_i u_i^{(t+1)}$:
\begin{align*}
\Lo(u_i^{(t+1)}, u_i^{(t+1)}, I) + \Et(u_i^{(t+1)}, u_i^{(t+1)}, I) & = \lambda_i u_i^{(t+1)} \\
\sum_{j=1}^{r} \sigma_i^* \left< u_i^{(t+1)}, u_j \right>^2 u_j + \Et(u_i^{(t+1)}, u_i^{(t+1)}, I) & = \lambda_i u_i^{(t+1)} \\
[\lambda_i I -  \Et(u_i^{(t+1)}, I, I)] u_i^{(t+1)} & = \sum_{j=1}^{r} \sigma_i^* \left< u_i^{(t+1)}, u_j \right>^2 u_j \\
u_i^{(t+1)} & = \left[I + \sum_{p \geq 1} \left( \frac{E^i}{\lambda_i} \right)^p \right] \sum_{j=1}^{r} \frac{\sigma_i^*}{\lambda_i} \left< u_i^{(t+1)}, u_j \right>^2 u_j 
\end{align*}
Now,
\begin{align*}
\|\Ltn-\Lo\|_\infty &\leq
\left\| \sum_{i\in [l]} \lambda_i (u_i^{(t+1)})^{\otimes 3} - \sum_{i\in [l]} \sigma_i^* {u_i}^{\otimes 3} \right\|_\infty + \left\| \sum_{i=l+1}^r \sigma_i^* {u_i}^{\otimes 3} \right\|_\infty \\
&\leq \sum_{i\in[l]} \left\| \lambda_i (u_i^{(t+1)})^{\otimes 3} - \sigma_i^* {u_i}^{\otimes 3} \right\|_\infty +  \sum_{i=l+1}^r \left\|  \sigma_i^* {u_i}^{\otimes 3} \right\|_\infty
\end{align*}
For a fixed $i$, using $\lambda_i \leq \sigma_i^* + \epsilon$~\cite{AnandkumarEtal:tensor12} and
using Lemma~\ref{lem:inf}, we obtain
\begin{align*}
\left\| \lambda_i (u_i^{(t+1)})^{\otimes 3} - \sigma_i^* {u_i}^{\otimes 3} \right\|_\infty & \leq \left\| (\sigma_i^* + \epsilon) (u_i^{(t+1)})^{\otimes 3} - \sigma_i^* {u_i}^{\otimes 3} \right\|_\infty \\
& \leq \left\| \sigma_i^* (u_i^{(t+1)})^{\otimes 3} - \sigma_i^* {u_i}^{\otimes 3} \right\|_\infty + \epsilon \left\| (u_i^{(t+1)})^{\otimes 3} \right\|_\infty \\
& \leq \sigma_i^* \left\| (u_i^{(t+1)})^{\otimes 3} - {u_i}^{\otimes 3} \right\|_\infty + \epsilon \left\| (u_i^{(t+1)})^{\otimes 3} \right\|_\infty \\
& \leq \sigma_i^* [ 3 \| u_i^{(t+1)} - u_i \|_\infty \| u_i \|_\infty^2 + 3 \| u_i^{(t+1)} - u_i \|_\infty^2 \| u_i \|_\infty + \| u_i^{(t+1)} - u_i \|_\infty^3 ] \\
& \; \; + \epsilon \| (u_i^{(t+1)})^{\otimes 3} \|_\infty \\
& \leq 7 \sigma_i^* \| u_i^{(t+1)} - u_i \|_\infty \| u_i \|_\infty^2 + \epsilon \| (u_i^{(t+1)}) \|_\infty^3
\end{align*}
Now,
\begin{align*}
\left\| u_i^{(t+1)} - u_i \right\|_\infty & = \left\| ( \sum_{j = 1}^{r} \frac{\sigma_i^*}{\lambda_i} \left< u_i^{(t+1)} , u_j \right>^2 u_j - u_i ) + \sum_{j = 1, p \geq 1}^{r} \frac{\sigma_i^*}{\lambda_i} \left< u_i^{(t+1)} , u_j \right>^2 ( E^i )^p u_j \right\|_\infty \\
& \leq \left\| (1 - \frac{\sigma_i^*}{\lambda_i} \left< u_i^{(t+1)} , u_i \right>^2) u_i \right\|_\infty + \left\| \sum_{j \neq i} \frac{\sigma_i^*}{\lambda_i} \left< u_i^{(t+1)} , u_j \right>^2 u_j \right\|_\infty \\
& + \left\| \sum_{p \geq 1} \frac{\sigma_i^*}{\lambda_i} \left< u_i^{(t+1)} , u_i \right>^2 \left( \frac{E^i}{\lambda_i} \right)^p u_i \right\|_\infty + \left\| \sum_{p, j \neq i} \frac{\sigma_i^*}{\lambda_i} \left< u_i^{(t+1)} , u_j \right>^2 \left( \frac{E^i}{\lambda_i} \right)^p u_j \right\|_\infty
\end{align*}
For the first term, we have
\begin{align*}
\left\| (1 - \frac{\sigma_i^*}{\lambda_i} \left< u_i^{(t+1)} , u_i \right>^2) u_i \right\|_\infty & \leq \left( 1 - \frac{\sigma_i^*}{\sigma_i^* + \epsilon} \left( 1 - \left( \frac{\epsilon}{\sigma_i^*} \right)^2 \right) \right) \| u_i \|_\infty \leq \left( 1 - \left( 1 - \frac{\epsilon}{\sigma_i^*} \right) \right) \frac{\mu}{n^{1/2}} \nonumber \\
&\leq \frac{\mu}{\sigma_i^* n^{1/2}} \epsilon \leq \frac{C\mu \sigma_l^*}{\sigma_i^* n}
\end{align*} where we substitute for $\epsilon$ in the last step.

For the second term, we have
\begin{align*}
\left\| \sum_{j \neq i} \frac{\sigma_i^*}{\lambda_i} \left< u_i^{(t+1)} , u_j \right>^2 u_j \right\|_\infty & \leq \frac{\sigma_i^*}{\sigma_i^* - \epsilon} \left( \frac{\epsilon}{\sigma_i^*} \right)^2 \| u_i \|_\infty \leq 2 \left( \frac{\epsilon}{\sigma_i^*} \right)^2 \frac{\mu}{n^{1/2}},
\end{align*} which is a lower order term.

Next,
\begin{align*}
\left\| \sum_{p \geq 1} \frac{\sigma_i^*}{\lambda_i} \left< u_i^{(t+1)} , u_i \right>^2 \left( \frac{E^i}{\lambda_i} \right)^p u_i \right\|_\infty & {\leq} \left\| \sum_{p \geq 1} \frac{\sigma_i^*}{\lambda_i} \left( \frac{E^i}{\lambda_i} \right)^p u_i \right\|_\infty {\leq} \sum_{p \geq 1} \frac{\sigma_i^*}{\lambda_i} \left\| \left( \frac{E^i}{\lambda_i} \right)^p u_i \right\|_\infty \\ &\leq \frac{\sigma_i^*}{\lambda_i}\cdot\frac{\mu}{\sqrt{n}} \cdot \frac{ \|\Et\|_2 \sqrt{n}/\lambda_i}{1- \|\Et\|_2 \sqrt{n}/\lambda_i} \\ 
& \leq \frac{2}{(1-C)}\frac{\kappa_t\mu}{\lambda_i\sqrt{n}} \|\Et\|_\infty
\end{align*} from Lemma~\ref{lem:sparse3}, and the assumption on spectral norm of $\|\Et\|_2$, where\[ \kappa_t:=\frac{Bd^2   \mu}{\sqrt{n}}.\]

For the remaining terms, we have
\begin{align*}
\left\| \sum_{p, j \neq i} \frac{\sigma_i^*}{\lambda_i} \left< u_i^{(t+1)} , u_j \right>^2 \left( \frac{E^i}{\lambda_i} \right)^p u_j \right\|_\infty & \leq \sum_{j \neq i} \frac{\sigma_i^*}{\lambda_i} \left< u_i^{(t+1)} , u_j \right>^2 \left\| \sum_{p \geq 1} \left( \frac{E^i}{\lambda_i} \right)^p u_1 \right\|_\infty \leq  \frac{\sigma_i^*}{\lambda_i} \left\| \sum_{p \geq 1} \left( \frac{E^i}{\lambda_i} \right)^p u_1 \right\|_\infty\left( \frac{\epsilon}{\sigma_i^*} \right)^2,
\end{align*} which is a lower order term.

Combining the above and recalling $\epsilon \ll \sigma_i^*, \; \forall i\in [l]$
\begin{align*}
\left\| u_i^{(t+1)} - u_i \right\|_\infty & \leq  \frac{8}{(1-C)}\frac{\kappa_t\mu}{\lambda_i\sqrt{n}} \|\Et\|_\infty .
\end{align*}
Also, from Lemma~\ref{lem:dist_1}
  $$  |\lambda_i - \sigma_i^* | \leq 8 \| \Et \|_2 \leq 8\epsilon$$
Thus, from the above two equations, we obtain the bound for the parameters (eigenvectors and eigenvalues) of the low-rank tensor $ \left\| u_i^{(t+1)} - u_i \right\|_\infty $ and $\left\| \lambda_i - \sigma_i^* \right\|_\infty$. We combine the individual parameter recovery bounds as:
\begin{align}\nn
\left\| \sum_{i\in [l]} \lambda_i (u_i^{(t+1)})^{\otimes 3} - \sum_{i\in [l]} \sigma_i^* {u_i}^{\otimes 3} \right\|_\infty & \leq r [ 7  \sigma_i^* \| u_i^{(t+1)} - u_i \|_\infty \| u_i \|_\infty^2 + \epsilon \| (u_i^{(t+1)}) \|_\infty^3 ] \\
&\leq \frac{224}{1-C} \frac{\kappa_t \mu^3r}{n^{1.5}} \|\Et\|_\infty\label{eqn:bound1}
\end{align}
and the other term
\[\sum_{i=l+1}^r \|\sigma_i^* u_i^{\otimes 3}\|_\infty \leq \sigma_{l+1}^* \frac{r \mu^3}{n^{1.5}}.\]
Combining bound in \eqref{eqn:bound1} with the above, we have
\[ \| \Ltn-\Lo\|_\infty \leq \frac{r \mu^3}{n^{1.5}}\left( \frac{224}{1-C} \kappa_t \|\Et\|_\infty   + \sigma_{l+1}^*\right) <\frac{1}{4} \|\Et\|_\infty.\]
where the last inequality comes from the fact that  $\frac{r \mu^3}{n^{1.5}} \sigma_{l+1}^*\leq \frac{\|\Et\|_\infty}{8}$ and the assumption that $C<1/2$, and we can choose $B$ and $d$ s.t.
\[     \frac{448 r \mu^3}{n^{1.5}} \kappa_t< \frac{1}{8}.\] This is possible from assumption (S).
\end{proof}

The following lemma bounds the support of $\Etn$ and $\infnorm{\Etn}$, using an assumption on $\infnorm{\Ltn-\Lo}$.
\begin{lemma}\label{lem:supp-bound}

Assume the notation of Lemma~\ref{lem:dec1}.
Suppose $$\infnorm{\Ltn-\Lo} \leq 2 \frac{\mu^3 r}{n^{3/2}}\left(\sigma_{l+1}^*+\left(\frac{1}{2}\right)^{t-1}\sigma_l^*\right).$$
Then, we have:
\begin{enumerate}
  \item	$\supp{\Etn} \subseteq \supp{\So}$.
  \item	$\infnorm{\Etn} \leq 7 \frac{\mu^3 r}{n^{3/2}}\left(\sigma_{l+1}^*+\left(\frac{1}{2}\right)^{t}\sigma_l^* \right)$, and

\end{enumerate}
\end{lemma}
\begin{proof}
We first prove the first conclusion. Recall that, $$\Stn=H_{\zeta}(\M-\Ltn)=H_{\zeta}(\Lo-\Ltn+\So),$$
where $\zeta=4 \frac{\mu^3 r}{n^{3/2}}\left(\lambda_{l+1}+\left(\frac{1}{2}\right)^{t}\lambda_l\right) $ is as defined
in Algorithm~\ref{algo:sap} and $\lambda_1,\cdots,\lambda_n$ are the eigenvalues of $\M-\St$ such that
${\lambda_1} \geq \cdots \geq {\lambda_n}$.

If $\So_{abc}=0$ then $\Etn_{ijk}=\mathbf{1}_{\abs{\Lo_{abc}-\Ltn_{abc}}>\zeta}\cdot (\Lo_{abc}-\Ltn_{abc})$. The first part of
the lemma now follows by using the assumption that $\infnorm{\Ltn-\Lo} \leq 2 \frac{\mu^3 r}{n^{3/2}}\left(\sigma_{l+1}^*
+\left(\frac{1}{2}\right)^{t}\sigma_l^* \right)\stackrel{(\zeta_1)}{\leq} 4 \frac{\mu^3 r}{n^{3/2}}\left(\lambda_{l+1}
+\left(\frac{1}{2}\right)^{t}\lambda_l \right)=\zeta$, where $(\zeta_1)$ follows from Lemma~\ref{lem:threshold_est}.

We now prove the second conclusion. We consider the following two cases: 
\begin{enumerate}
  \item  $\abs{\M_{abc}-\Ltn_{abc}} > \zeta$:
	Here, $\Stn_{abc}=\So_{abc}+\Lo_{abc}-\Ltn_{abc}$. Hence, $|\Stn_{abc}-\So_{abc}|\leq |\Lo_{abc}-\Ltn_{abc}|\leq 2 \frac{\mu^3 r}{n^{3/2}}\left(\sigma_{l+1}^*+\left(\frac{1}{2}\right)^{t}\sigma_l^* \right)$.
  \item $\abs{\M_{abc}-\Ltn_{abc}} \leq  \zeta $:
	In this case, $\Stn_{abc}=0$ and $\abs{\So_{abc}+\Lo_{abc}-\Ltn_{abc}} \leq \zeta$.
	So we have, $\abs{\Etn_{abc}}=\abs{\So_{abc}} \leq \zeta+\abs{\Lo_{abc}-\Ltn_{abc}} \leq 7 \frac{\mu^3 r}{n^{3/2}}\left(\sigma_{l+1}^*+\left(\frac{1}{2}\right)^{t}\sigma_l^* \right)$.
	The last inequality above follows from Lemma~\ref{lem:threshold_est}.
\end{enumerate}
This proves the lemma.
\end{proof}

\begin{theorem}
Let $\Lo, \So$ be symmetric and satisfy $(L)$ and $(S)$, and $\beta= 4 \frac{\mu^3 r}{n^{3/2}}$. The outputs $\widehat{\L}$ (and its parameters $\hat{u}_i$ and $\hat{\lambda}_i$) and $\widehat{\S}$ of Algorithm~\ref{algo:sap} satisfy w.h.p.:
\begin{align*}
  \left\| \hat{u}_i - u_i \right\|_\infty \leq \frac{\delta}{\mu^2 r n^{1/2} \sigma_{\min}^*}, \ \
  | \hat{\lambda}_i - \sigma_i^* | \leq \delta, \ \ \ \forall i \in [n] , \\
\frob{\widehat{\L} - \Lo} \leq \delta, \ \ \infnorm{\widehat{\S} - \So} \leq \frac{\delta}{n^{3/2}}, \mbox{ and }\ \ \supp{\widehat{\S}} \subseteq \supp{\So}.
\end{align*}
\label{thm:main-appendix}
\end{theorem}
\begin{proof}
Recall that in the $l^{\textrm{th}}$ stage, the update $\Ltn$ is given by: $\Ltn=P_l(\M-\St)$ and $\Stn$ is given by: $\Stn=H_\zeta(\M-\Ltn)$.
Also, recall that $\Et \defas \So-\St$ and $\Etn \defas \So-\Stn$.

We prove the lemma by induction on both $l$ and $t$.
For the base case ($l=1$ and $t=-1$), we first note that the first inequality on $\infnorm{\Lt[0]-\Lo}$ is trivially satisfied.
Due to the thresholding step (step $3$ in Algorithm~\ref{algo:sap}) and the incoherence assumption on $\Lo$, we have:
\begin{align*}
  \infnorm{\E^{(0)}} &\leq \frac{8 \mu^3 r}{n^{3/2}}\left(\sigma_2^* + 2 \sigma_1^*\right), \mbox{ and }\\
  \supp{\Et[0]} &\subseteq \supp{\So}.
\end{align*}
So the base case of induction is satisfied.

We first do the inductive step over $t$ (for a fixed $r$).
By inductive hypothesis we assume that:
a) $\infnorm{\Et} \leq \frac{8 \mu^3 r}{n^{3/2}}\left(\sigma_{l+1}^*+\left(\frac{1}{2}\right)^{t-1}\sigma_l^*
\right)$,
b) $\supp{\Et}\subseteq \supp{\So}$. Then by Lemma~\ref{lem:dec1}, we have:
\begin{align*}
\infnorm{\Ltn-\Lo} \leq \frac{2 \mu^3 r}{n^{3/2}} \left(\sigma_{l+1}^*+\left(\frac{1}{2}\right)^{t}\sigma_l^*
\right).
\end{align*}
Lemma~\ref{lem:supp-bound} now tells us that
\begin{enumerate}
  \item	$\infnorm{\Etn} \leq  \frac{8 \mu^3 r}{n^{3/2}}\left(\sigma_{l+1}^*+\left(\frac{1}{2}\right)^{t}\sigma_l^*  \right)$, and
  \item	$\supp{\Etn} \subseteq \supp{\So}$.
\end{enumerate}
This finishes the induction over $t$. Note that we show a stronger bound than necessary on $\infnorm{\Etn}$.

We now do the induction over $l$. Suppose the hypothesis holds for stage $l$.
Let $T$ denote the number of iterations in each stage. We first obtain a lower bound on $T$.
Since
\begin{align*}
\twonorm{\M-\S^{(0)}} \geq \twonorm{\Lo } - \twonorm{\E^{(0)}} \geq \sigma_1^* - d^{3/2} \infnorm{\E^{(0)}}
\geq \frac{3}{4} \sigma_1^*,
\end{align*}
we see that $T \geq 10 \log \left(3 \mu^3 r \sigma_1^*/\delta\right)$.
So, at the end of stage $r$, we have:
\begin{enumerate}
  \item	$\infnorm{\E^{(T)}} \leq \frac{7 \mu^3 r}{n^{3/2}}\left(\sigma_{l+1}^*+\left(\frac{1}{2}\right)^{T}\sigma_l^*
   \right)
	\leq \frac{7 \mu^3 r {\sigma_{l+1}^*}}{n^{3/2}} + \frac{\delta}{10 n}$, and
  \item	$\supp{\Et[T]} \subseteq \supp{\So}$.
\end{enumerate}
Recall,
$\abs{\sigma_{r+1}\left(\M-\St[T]\right) - {\sigma_{r+1}^*}} \leq \twonorm{\E^{(T)}} \leq \frac{d}{n} \left( \mu^3 r \abs{\sigma_{r+1}^*} + \delta \right)$.
We will now consider two cases:
\begin{enumerate}
  \item	\textbf{Algorithm~\ref{algo:sap} terminates:} This means that $\beta \sigma_{r+1}\left(\M-\St[T]\right) <
\frac{\delta}{2 n^{3/2}}$ which then implies that ${\sigma_{r+1}^*} < \frac{\delta}{6 \mu^3 r}$. So we have:
\begin{align*}
\infnorm{\widehat{L} - \Lo} = \infnorm{\Lt[T] - \Lo}%
	\leq \frac{2 \mu^3 r}{n^{3/2}} \left(\sigma_{r+1}^*+\left(\frac{1}{2}\right)^{T}\sigma_r^*
	\right) \leq \frac{\delta}{5n^{3/2}}.
\end{align*}
This proves the statement about $\widehat{L}$ and its parameters (eigenvalues and eigenvectors). A similar argument proves the claim on $\infnorm{\widehat{S}-\So}$.
The claim on $\supp{\widehat{S}}$ follows since $\supp{\Et[T]} \subseteq \supp{\So}$.
  \item	\textbf{Algorithm~\ref{algo:sap} continues to stage $\bm{\left(r+1\right)}$:} This means that
$\beta \sigma_{r+1}\left(\Lt[T]\right) \geq \frac{\delta}{2 n^{3/2}}$ which then implies that
${\sigma_{r+1}^*} > \frac{\delta}{8 \mu^3 r}$. So we have:
\begin{align*}
\infnorm{\E^{(T)}} &\leq \frac{8 \mu^3 r}{n^{3/2}}\left(\sigma_{r+1}^*+\left(\frac{1}{2}\right)^{T}\sigma_r^*
\right) \\
	&\leq \frac{8 \mu^3 r}{n^{3/2}}\left(\sigma_{l+1}^*+ \frac{\delta}{10\mu^3 r n^{3/2}}
\right) \\
	&\leq \frac{8 \mu^3 r}{n^{3/2}}\left(\sigma_{l+1}^*+ \frac{8 {\sigma_{l+1}^*}}{10 n}
	 \right) \\
	&\leq \frac{8 \mu^3 r}{n^{3/2}}\left(\sigma_{l+2}^*+ 2 {\sigma_{l+1}^*}
	\right).
\end{align*}
Similarly for $\infnorm{\Lt[T]-\Lo}$.
\end{enumerate}
This finishes the proof.
\end{proof}

\subsection{Short proof of Corollary~\ref{cor:sup}}\label{proof:cor}
The state of art guarantees for robust matrix PCA requires that the overall sparsity along any row or column   of the input matrix   be $D=O(\frac{n}{r \mu^2})$ (when the input matrix is $\R^{n \times n}$).

Under (S), the total sparsity along any row or column of $M_i$   is given by $D:=d B $. Now, Theorem~\ref{thm:main} holds when the sparsity condition in \eqref{eqn:S1} is satisfied. That is, $\ncralgo$ succeeds when $$D=O( d \cdot B)=O\left( \min\left(\frac{n^{4/3}}{r^{1/3}\mu^2}, \frac{n^{2/3}}{r^{2/3}\mu^2}(\frac{n}{r})^{1/3}\right)\right)=O\left(\frac{n}{r\mu^2}\right).$$
Hence, $\ncralgo$ can handle larger amount of corruption than the matrix methods and the gain becomes more significant for smaller $\eta$.

\subsection{Some auxiliary lemmas}
We recall Theorem 5.1 from \cite{AnandkumarEtal:tensor12}. Let $\epsilon = 8 \| \Et \|_2$ where $\Et := \So - \St$.
\begin{lemma}
\label{lem:dist}
Let $\Ltn = \sum_{i=1}^{k} \lambda_i u_i^{(t+1)}$ be the eigen decomposition obtained using Algorithm~\ref{algo:sspm} on $(T - \St)$. Then,
\begin{enumerate}
\item \label{lem:dist_1} If $\| u_i^{(t+1)} - u_i \|_2 \leq \frac{\epsilon}{\sigma_{\min}^*}$, then $\dist (u_i^{(t+1)}, u_i) \leq \frac{\epsilon}{\sigma_{\min}^*}$.
\item \label{lem:dist_2} $\sum_{j \neq i} \left< u_i^{(t+1)} , u_j \right>^2 \leq \left( \frac{\epsilon}{\sigma_{\min}^*} \right)^2$.
\item \label{lem_dist_3} $\| u_i^{(t+1)} \|_\infty \leq \frac{\mu}{n^{1/2}} + \frac{\epsilon}{\sigma_{\min}^*}$.
\item $ |\sigma_i^*| - \epsilon \leq | \lambda_i | \leq |\sigma_i^*| + \epsilon$.
\end{enumerate}
\end{lemma}
\begin{proof}
\begin{enumerate}
\item Let $z \perp u $ and $\| z \|_2 = 1$.
\begin{align*}
u_i^{(t+1)} & = \left< u_i^{(t+1)}, u_i \right> u_i + \dist(u_i^{(t+1)}, u_i) z \\
\| u_i^{(t+1)} - u_i \|_2^2 & = (\left< u_i^{(t+1)}, u_i \right> - 1)^2 \| u_i \|_2^2 + \dist(u_i^{(t+1)}, u_i) \| z \|_2^2 + 0 \\
& \geq (\dist( u_i^{(t+1)}, u_i ))^2
\end{align*}
Then using Theorem 5.1 from \cite{AnandkumarEtal:tensor12}, we obtain the result. Next, since $\inner{u_i^{(t+1)}, u_i}^2 + \dist(u_i^{(t+1)}, u_i)^2 = 1$, we have $\inner{u_i^{(t+1)}, u_i}^2 \geq 1-\left( \frac{\epsilon}{\sigma_{\min}^*} \right)^2$.
\item Note that
$$u_i^{(t+1)} = \sum_{j=1}^{k} \left< u_i^{(t+1)} , u_j \right> u_j + \dist(u_i^{(t+1)}, U) z$$
where $z \perp U$ such that $\| z \|_2 = 1$. Using $ \| u_i^{(t+1)} \|_2 = 1$ and the Pythagoras theorem, we get
$$1 - \left< u_i^{(t+1)} , u_i \right>^2 = \sum_{j \neq i} \left< u_i^{(t+1)} , u_j \right>^2 + \dist(u_i^{(t+1)}, U)^2 . 1 \geq \sum_{j \neq i} \left< u_i^{(t+1)} , u_j \right>^2$$
Using part 1 of Lemma~\ref{lem:dist}, we get $\sum_{j \neq i} \left< u_i^{(t+1)} , u_j \right>^2 \leq \left( \frac{\epsilon}{\sigma_{\min}^*} \right)^2$.
\item We have
\begin{align*}
u_i^{(t+1)} & = \left< u_i^{(t+1)}, u_i \right> u_i + \dist (u_i^{(t+1)}, u_i) z \\
\| u_i^{(t+1)} \|_\infty & \leq | \left< u_i^{(t+1)}, u_i \right> | \| u_i \|_\infty + | \dist (u_i^{(t+1)}, u_i) | \| z \|_\infty \leq 1. \frac{\mu}{n^{1/2}} + \frac{\epsilon}{\sigma_{\min}^*}
\end{align*}
\item This follows from Theorem 5.1 from \cite{AnandkumarEtal:tensor12}, i.e., $\forall i, \; | |\lambda_i| - |\sigma_i^*|| \leq \epsilon$.
\end{enumerate}
\end{proof}

\begin{lemma}
\label{lem:inf}
Let $a = b + \epsilon . \overrightarrow{1}$ where $a, b$ are any $2$ vectors and $\epsilon > 0$. Then, $\|
 a^{\otimes 3} - b^{\otimes 3} \|_\infty \leq \| a - b \|_\infty . \| b \|_\infty^2 + O(\epsilon^2)$.
\end{lemma}
\begin{proof}
We have
$$\| a^{\otimes 3} - b^{\otimes 3} \|_\infty = \| (b+\epsilon \overrightarrow{1})^{\otimes 3} - b^{\otimes 3} \|_\infty$$
Let $(i, j, k)$ be the maximum element. Therefore,
\begin{align*}
\| (b+\epsilon \overrightarrow{1})^{\otimes 3} - b^{\otimes 3} \|_\infty & = (b_i + \epsilon)(b_j + \epsilon)(b_k + \epsilon) - b_i b_j b_k \\
& = \epsilon (b_i b_j + b_j b_k + b_k b_i) + \epsilon^2 (b_i + b_j + b_k) + \epsilon^3
\end{align*}
With $b_i \leq c$ $\forall i$ for some $c > 0$ and $\epsilon = \| a - b \|_\infty$, we have $\|
 a^{\otimes 3} - b^{\otimes 3} \|_\infty \leq 3 \epsilon c^2 + O(\epsilon^2)$
\end{proof}

\section{Symmetric embedding of an asymmetric tensor}
We use the symmetric embedding $sym(L)$ of a tensor $L$ as defined in Section 2.3 of ~\cite{ragnarsson2013block}. We focus on third order tensors which have low CP-rank. We have three properties to derive that is relevant to us:
\begin{enumerate}
\item \textit{Symmetry: }From Lemma 2.2 of ~\cite{ragnarsson2013block} we see that $sym(L)$ for any tensor is symmetric.
\item \textit{CP-Rank: }From Equation 6.5 of ~\cite{ragnarsson2013block} we see that CP-rank($sym(L)$) $\leq$ $6$.CP-rank($L$). Since this is a constant, we see that the symmetric embedding is also a low-rank tensor.
\item \textit{Incoherece: }Theorem 4.7 of ~\cite{ragnarsson2013block} says that if $u_1$, $u_2$ and $u_3$ are unit modal singular vectors of $T$, then the vector $\tilde{u} = 3^{-1/2} [u_1;u_2;u_3]$ is a unit eigenvector of $sym(T)$. Without loss of generality, assume that $T$ is of size $n_1 \times n_2 \times n_3$ with $n_1 \leq n_2 \leq n_3$. In this case, we have
\begin{align}
\| \tilde{u} \|_\infty \leq \frac{\mu}{(3 n_1)^{1/2}}
\label{inc_1}
\end{align}
and
\begin{align}
\| \tilde{u} \|_\infty \leq \frac{\tilde{\mu}}{(n_1 + n_2 + n_3)^{1/2}}
\label{inc_2}
\end{align}
for $\tilde{\mu} = c \mu$ for some constant $c$ to be calculated. Equating the right hand sides of Equations~\eqref{inc_1} and ~\eqref{inc_2}, we obtain $c = [ (n_1 + n_2 + n_3)/(3 n_1) ]^{1/2}$. When $\Theta(n_1)=\Theta(n_2)=\Theta(n_3)$, we see that the eigenvectors $\tilde{u}$ of $sym(T)$ as specified above have the incoherence-preserving property.
\end{enumerate}

\section{Proof of Theorem~\ref{thm:robustpower}}
\label{app:power}
Let $\widetilde{L}$ be a symmetric tensor which is a perturbed version of an orthogonal tensor $\Lo$, $ \widetilde{L} = \Lo+ E \in \R^{n \times n \times n}, \quad \Lo = \sum_{i\in [r]} \sigma^*_i
u_i^{\otimes 3},$  where  $\sigma^*_1 \geq \sigma^*_2\ldots \sigma^*_r > 0$ and $\{ u_1, u_2, \dotsc, u_r \}$
form an orthonormal basis.

The analysis proceeds iteratively. First, we prove convergence to  eigenpair of $\widetilde{L}$, which is close to top eigenpair  $(\sigma^*_1,u_1)$ of $\Lo$. We then argue that the same holds on the deflated tensor, when the perturbation $E$ satisfies \eqref{eqn:condpower}. from This finishes the proof of Theorem~\ref{thm:robustpower}.

To prove convergence for the first stage, i.e. convergence to  eigenpair of $\widetilde{L}$, which is close to top eigenpair  $(\sigma^*_1,u_1)$ of $\Lo$, we analyze two phases of the shifted power iteration. In the first phase, we prove  that with $N_1$ initializations and $N_2$ power iterations, we get close to true top eigenpair of $\Lo$, i.e. $(\sigma^*_1,u_1)$. After this, in the second phase, we prove  convergence to an eigenpair of $\widetilde{L}$.

The proof of the second phase is outlined in the main text. Here, we now provide proof for the first phase.

\subsection{Analysis of first phase of shifted power iteration}
In this section, we prove that the output of shifted power method is close to original eigenpairs of the (unperturbed) orthogonal tensor, i.e. Theorem~\ref{thm:robustpower} holds, except for the property that the output corresponds to the eigenpairs of the perturbed tensor.
We adapt the proof of tensor power iteration from~\cite{AnandkumarEtal:tensor12} but here, since we consider the shifted power method, we need to modify it.  We adopt the notation of~\cite{AnandkumarEtal:tensor12} in this section.

Recall the update rule used in the shifted power method.
Let $\th{t} = \sum_{i=1}^k \th{i,t} v_i \in \R^k$ be the unit vector at
time $t$.
Then
\begin{align*}
\th{t+1} = \sum_{i=1}^k \th{i,t+1} v_i
& := (\tilde{T}(I,\th{t},\th{t}) + \alpha \theta_t )/ \| (\tilde{T}(I,\th{t},\th{t}) + \alpha \theta_t) \| .
\end{align*}

In this subsection, we assume that $\tilde{T}$ has the form
\begin{equation} \label{eq:tildeT}
\tilde{T} = \sum_{i=1}^k \tlambda_i
v_i^{\otimes 3} + \tilde{E}
\end{equation}
where $\{ v_1, v_2, \dotsc, v_k \}$ is an orthonormal basis, and, without
loss of generality,
\[ \tlambda_1 |\th{1,t}| = \max_{i \in [k]} \tlambda_i |\th{i,t}| > 0 . \]
Also, define
\[ \tlambdamin := \min \{ \tlambda_i : i \in [k], \ \tlambda_i > 0 \} ,
\quad \tlambdamax := \max \{ \tlambda_i : i \in [k] \} . \]

We   assume the error $\tilde{E}$ is a symmetric tensor such that,
for some constant $p > 1$,
\begin{align}
\|\tilde{E}(I,u,u)\| & \leq \teps , \quad \forall u \in S^{k-1} ;
\label{eq:regular-err}
\\
\|\tilde{E}(I,u,u)\| & \leq \teps / p , \quad \forall u \in S^{k-1} \
\text{s.t.} \ (u^\t v_1)^2 \geq 1 - (3\teps/\tlambda_1)^2
.
\label{eq:smaller-err}
\end{align}

In the next two propositions (Propositions~\ref{prop:simple}
and~\ref{prop:one-step}) and  Lemmas~\ref{lem:r}, we analyze the power method iterations using $\tilde{T}$
at some arbitrary iterate $\th{t}$ using only the
property~\eqref{eq:regular-err} of $\tilde{E}$.
But throughout, the quantity $\teps$ can be replaced by $\teps/p$ if
$\th{t}$ satisfies $(\th{t}^\t v_1)^2 \geq 1 - (3\teps/\tlambda_1)^2$ as
per property~\eqref{eq:smaller-err}.

Define
\begin{equation} \label{eq:defns}
\begin{aligned}
R_{\tau} & := \biggl( \frac{\th{1,\tau}^2}{1 - \th{1,\tau}^2} \biggr)^{1/2} ,
& r_{i,\tau} & := \frac{\tlambda_1\th{1,\tau}}{\tlambda_i |\th{i,\tau}|} ,
\\
\gamma_{\tau} & := 1 - \frac1{\min_{i\neq1}|r_{i,\tau}|} ,
& \delta_{\tau} & := \frac {\teps} {\tlambda_1 \th{1,\tau}^2} ,
& \kappa & := \frac{\tlambdamax}{\tlambda_1}
\end{aligned}
\end{equation}
for $\tau \in \{t,t+1\}$.

\begin{prop} \label{prop:simple}
\begin{align*}
\min_{i\neq1} |r_{i,t}| & \geq \frac{R_t}{\kappa} , &
\gamma_t & \geq 1 - \frac{\kappa}{R_t} , &
\th{1,t}^2 & = \frac{R_t^2}{1+R_t^2}
.
\end{align*}
\end{prop}

\begin{prop} \label{prop:one-step}
\begin{align}
r_{i,t+1}
& \geq r_{i,t}^2 \cdot \frac {1 - \delta_t + \frac{\alpha}{\tilde{\lambda}_1 \th{1,t}}}
{1 + \kappa \delta_t r_{i,t}^2 + \frac{\alpha}{\tilde{\lambda}_i \th{i,t}}}
, \quad i \in [k] ,
\label{eq:ratio-ineq}
\\
R_{t+1}
& \geq = R_t \cdot \frac{1-\delta_t+ \frac{\alpha}{\tilde{\lambda}_1 |\th{1,t}|}} {1-\gamma_t + \left( \delta_t  + \frac{\alpha (1 - \th{1,t})^{1/2}}{\tilde{\lambda}_1 \th{1,t}^2} \right) R_t}
\geq \frac{1-\delta_t+ \frac{\alpha}{\tilde{\lambda}_1 |\th{1,t}|}} {\frac{\kappa}{R_t^2} +
\delta_t + \frac{\alpha (1 - \th{1,t})^{1/2}}{\tilde{\lambda}_1 \th{1,t}^2}}
.
\label{eq:energy-ineq}
\end{align}
\end{prop}
\begin{proof}
Let $\ut{t+1} := \tilde{T}(I,\th{t},\th{t}) + \alpha \theta_t$, so $\th{t+1} = \ut{t+1} /
\|\ut{t+1}\|$.
Since $\ut{i,t+1} = \tilde{T}(v_i,\th{t},\th{t}) = T(v_i,\th{t},\th{t}) + \alpha \theta_t +
E(v_i,\th{t},\th{t})$, we have
\begin{equation*}
\ut{i,t+1} = \tlambda_i \th{i,t}^2 + E(v_i,\th{t},\th{t}) + \alpha \theta_t^\top v_i, \quad i \in [k]
.
\end{equation*}
By definition, we have $\th{i,t} = \theta_t^\top v_i$. Using the triangle inequality and the fact $\|E(v_i,\th{t},\th{t})\| \leq
\teps$, we have
\begin{equation} \label{eq:coord-lb}
\ut{i,t+1}
\geq \tlambda_i \th{i,t}^2 - \teps + \alpha \th{i, t}
\geq |\th{i,t}| \cdot \Bigl( \tlambda_i |\th{i,t}| - \teps / |\th{i,t}| + \alpha
\Bigr)
\end{equation}
and
\begin{equation} \label{eq:coord-ub}
|\ut{i,t+1}|
\leq |\tlambda_i \th{i,t}^2| + \teps + \alpha \th{i,t}
\leq |\th{i,t}| \cdot \Bigl( \tlambda_i |\th{i,t}| + \teps / |\th{i,t}| + \alpha
\Bigr)
\end{equation}
for all $i \in [k]$.
Combining \eqref{eq:coord-lb} and \eqref{eq:coord-ub} gives
\[
r_{i,t+1}
= \frac {\tlambda_1\th{1,t+1}} {\tlambda_i|\th{i,t+1}|}
= \frac {\tlambda_1\ut{1,t+1}} {\tlambda_i|\ut{i,t+1}|}
\geq
r_{i,t}^2 \cdot \frac {1 - \delta_t + \frac{\alpha}{\tilde{\lambda}_1 \th{1,t}}} {1 +
\frac{\teps}{\tlambda_i\th{i,t}^2} + \frac{\alpha}{\tilde{\lambda}_i \th{i,t}}}
=
r_{i,t}^2 \cdot \frac {1 - \delta_t + \frac{\alpha}{\tilde{\lambda}_1 \th{1,t}}}
{1 + (\tlambda_i/\tlambda_1) \delta_t r_{i,t}^2 + \frac{\alpha}{\tilde{\lambda}_i \th{i,t}}}
\geq
r_{i,t}^2 \cdot \frac {1 - \delta_t + \frac{\alpha}{\tilde{\lambda}_1 \th{1,t}}}
{1 + \kappa \delta_t r_{i,t}^2 + \frac{\alpha}{\tilde{\lambda}_i \th{i,t}}}
.
\]

Moreover, by the triangle inequality and H\"older's inequality,
\begin{align}
\biggl( \sum_{i=2}^n [\ut{i,t+1}]^2 \biggr)^{1/2}
& = \biggl( \sum_{i=2}^n \Bigl( \tlambda_i \th{i,t}^2 + E(v_i,\th{t},\th{t}) + \alpha \th{i,t}
\Bigr)^2 \biggr)^{1/2}
\nonumber \\
& \leq \biggl( \sum_{i=2}^n \tlambda_i^2 \th{i,t}^4 \biggr)^{1/2}
+ \biggl( \sum_{i=2}^n E(v_i,\th{t},\th{t})^2 \biggr)^{1/2}
+ \biggl( \sum_{i=2}^k \alpha^2 \th{i,t}^2 \biggr)^{1/2}
\nonumber \\
& \leq \max_{i\neq1} \tlambda_i |\th{i,t}|
\biggl( \sum_{i=2}^n \th{i,t}^2 \biggr)^{1/2} + \teps + \biggl( \alpha^2 \sum_{i=2}^k \th{i,t}^2 \biggr)^{1/2}
\nonumber \\
& = (1 - \th{1,t}^2)^{1/2} \cdot \Bigl( \max_{i\neq1} \tlambda_i |\th{i,t}|
+ \teps / (1 - \th{1,t}^2)^{1/2} + \alpha \Bigr)
.
\label{eq:other-coord-ub}
\end{align}
Combining \eqref{eq:coord-lb} and \eqref{eq:other-coord-ub} gives
\[
\frac {|\th{1,t+1}|} {(1 - \th{1,t+1}^2)^{1/2}}
=
\frac {|\ut{1,t+1}|} {\Bigl( \sum_{i=2}^n [\ut{i,t+1}]^2 \Bigr)^{1/2}}
\geq
\frac{|\th{1,t}|}{(1 - \th{1,t}^2)^{1/2}}
\cdot
\frac {\tlambda_1 |\th{1,t}| - \teps / |\th{1,t}| + \alpha} {\max_{i\neq1} \tlambda_i
|\th{i,t}| + \teps / (1 - \th{1,t}^2)^{1/2} + \alpha }
.
\]
In terms of $R_{t+1}$, $R_t$, $\gamma_t$, and $\delta_t$, this reads
\begin{align*}
R_{t+1} & \geq \frac {1-\delta_t + \frac{\alpha}{\tilde{\lambda}_1 |\th{1,t}|}}
{(1-\gamma_t) \Bigl( \frac{1-\th{1,t}^2}{\th{1,t}^2} \Bigr)^{1/2} +
\delta_t + \frac{\alpha (1 - \th{1,t})^{1/2}}{\tilde{\lambda}_1 \th{1,t}^2}}
= R_t \cdot \frac{1-\delta_t+ \frac{\alpha}{\tilde{\lambda}_1 |\th{1,t}|}} {1-\gamma_t + \left( \delta_t  + \frac{\alpha (1 - \th{1,t}^2)^{1/2}}{\tilde{\lambda}_1 \th{1,t}^2} \right) R_t} \\
& = \frac{1-\delta_t+ \frac{\alpha}{\tilde{\lambda}_1 |\th{1,t}|}} {\frac{1-\gamma_t}{R_t} + \left( \delta_t  + \frac{\alpha (1 - \th{1,t}^2)^{1/2}}{\tilde{\lambda}_1 \th{1,t}^2} \right) }
\geq \frac{1-\delta_t+ \frac{\alpha}{\tilde{\lambda}_1 |\th{1,t}|}} {\frac{\kappa}{R_t^2} +
\delta_t + \frac{\alpha (1 - \th{1,t}^2)^{1/2}}{\tilde{\lambda}_1 \th{1,t}^2}}
\end{align*}
where the last inequality follows from Proposition~\ref{prop:simple}.
\end{proof}

\begin{lem}\label{lem:r}
Fix any $\rho > 1$.
Assume
\[ 0 \leq \delta_t < \min\Bigl\{ \frac1{2(1+2\kappa\rho^2)}, \
\frac{1-1/\rho}{1+\kappa\rho} \Bigr\} \]
and $\gamma_t > 2(1+2\kappa\rho^2) \delta_t$.
\begin{enumerate}
\item If $r_{i,t}^2 \leq 2\rho^2$, then $r_{i,t+1} \geq |r_{i,t}| \bigl( 1
+ \frac{\gamma_t}{2} \bigr)$.
\end{enumerate}
\end{lem}
\begin{proof}
By~\eqref{eq:ratio-ineq} from Proposition~\ref{prop:one-step},
\[
r_{i,t+1}
\geq r_{i,t}^2 \cdot \frac{1-\delta_t+ \frac{\alpha}{\tilde{\lambda}_1 \th{1,t}}}{1+\kappa\delta_tr_{i,t}^2+ \frac{\alpha}{\tilde{\lambda}_i \th{i,t}}}
\geq |r_{i,t}| \cdot \frac{1}{1-\gamma_t} \cdot
\frac{1-\delta_t + \frac{\alpha}{\tilde{\lambda}_1 \th{1,t}}}{1+2\kappa\rho^2\delta_t + \frac{\alpha}{\tilde{\lambda}_i \th{i,t}}}
\geq |r_{i,t}| \Bigl( 1 + \frac{\gamma_t}{2} \Bigr)
\]
where the last inequality is seen as follows:
Let $$\xi = 2.\frac{1-\delta_t+\frac{\alpha}{\tilde{\lambda}_1 \th{1,t}}}{1+2\kappa \rho^2 \delta_t+\frac{\alpha}{\tilde{\lambda}_i \th{i,t}}}$$
Then, we have
$\gamma_t^2 + \gamma_t -2 + \xi \geq 0$. The positive root is $\frac{-1 + (9-4 \xi)^{1/2}}{2}$. Since $\gamma_t \geq 0$, we have $(9-4 \xi)^{1/2} \geq 1$, so we assume $\xi \leq 2$ for the inequality to hold, i.e., $\frac{\alpha}{\tilde{\lambda}_1 \th{1,t}} - \frac{\alpha}{\tilde{\lambda}_i \th{i,t}} \leq (1+2 \kappa \rho^2) \delta_t$.
\end{proof}

The rest of the proof is along the similar lines of~\cite{AnandkumarEtal:tensor12}, except that we use SVD initialization instead of random initialization. The proof of SVD initialization is given in~\cite{DBLP:journals/corr/AnandkumarGJ14}.

\end{document}